\documentclass[11pt]{article}

\usepackage[top=1in,bottom=1in,right=1in,left=1in]{geometry}
\usepackage{amssymb,amsmath,amsthm}
\usepackage{arydshln}
\usepackage{verbatim}
\usepackage{titlesec}
\usepackage{subcaption}
\usepackage[linesnumbered,ruled]{algorithm2e}
\usepackage{bm}
\usepackage{pdfpages}
\usepackage[round]{natbib}
\usepackage{wrapfig}
\usepackage{tikz}
\usetikzlibrary{trees,calc}

\definecolor{carminepink}{rgb}{0.92, 0.3, 0.26}
\definecolor{burntsienna}{rgb}{0.91, 0.45, 0.32}

\definecolor{forestgreen}{rgb}{0.13, 0.55, 0.13}
\usepackage[colorlinks, linkcolor=blue,citecolor=forestgreen,urlcolor = black, bookmarks=true]{hyperref}

\usepackage[T1]{fontenc}
 
\usepackage[nameinlink, noabbrev, capitalize]{cleveref}
\usepackage{times}

\allowdisplaybreaks

\newcommand{\Var}{\mathrm{Var}}
\newcommand{\Kurt}{\mathrm{Kurt}}

\crefname{equation}{}{}
\crefrangeformat{equation}{(#3#1#4) --~(#5#2#6)}
\crefmultiformat{equation}{(#2#1#3)}{, (#2#1#3)}{, (#2#1#3)}{, (#2#1#3)}
\crefname{lem}{Lemma}{Lemmas}
\crefname{section}{Section}{Sections}
\crefname{subsubsubsection}{Section}{Sections}
\crefname{rem}{Remark}{Remarks}
\crefname{cor}{Corollary}{Corollaries}
\crefname{figure}{Figure}{Figures}
\crefname{table}{Table}{Tables}
\Crefname{lem}{Lemma}{Lemmas}
\Crefname{line}{Line}{Lines}
\Crefname{fact}{Fact}{Facts}
\crefname{thm}{Theorem}{Theorems}
\crefname{assumption}{Assumption}{Assumptions}

\usepackage{commath}

\newtheorem{thm}{Theorem}[section]

\newtheorem{defn}[thm]{Definition}
\newtheorem{lem}[thm]{Lemma}

\newtheorem{prop}[thm]{Proposition}

\newtheorem{rmk}{Remark}[section]

\title{A Hierarchical Language Model with Predictable Scaling Laws and Provable Benefits of Reasoning}
\usepackage[affil-it]{authblk}

\author[1]{Jason Gaitonde}
\author[2]{Frederic Koehler}
\author[3]{Elchanan Mossel}
\author[2]{Joonhyung Shin}
\author[4]{Allan Sly}

\affil[1]{Duke University}
\affil[2]{University of Chicago}
\affil[3]{Massachusetts Institute of Technology}
\affil[4]{Princeton University}
\begin{document}

\maketitle

\begin{abstract}
We introduce a family of synthetic languages with hierarchical structure --- generated by a broadcast process on trees --- for which the role of context length and reasoning in autoregressive generation can be analyzed precisely. At the heart of our analytic approach is an \emph{exact $k$-gram ansatz} in place of transformers with context length $k$, a substitution we then validate empirically. Using this ansatz we derive explicit asymptotic predictions for distributional statistics of the sequences produced by a trained model, instantiated in two settings. For the \emph{Ising broadcast process} (a soft-constrained language), we prove that the variance of the generated sum scales log-linearly in the context depth and its kurtosis converges to that of a Gaussian --- both deviating from the true language for any sublinear context. For the \emph{coloring broadcast process} (a hard-constrained language) in the freezing regime, bounded-context autoregression produces sequences that, with high probability, are inconsistent with \emph{any} valid coloring of the underlying tree. Together these results imply an $\Omega(n)$ lower bound on the context length required to faithfully sample length-$n$ sequences. In contrast, we prove that an autoregressive \emph{reasoning} model with only $\Theta(\log n)$ working memory can sample exactly from the true language --- an exponential improvement. We confirm both the lower-bound predictions and the reasoning-based upper bound empirically with transformers trained on the synthetic language; the trained models track our asymptotic predictions quantitatively across a wide range of context sizes.
\end{abstract}

\section{Introduction}

The size of an autoregressive language model's context window is one of the central design parameters in modern LLMs. Larger contexts allow more long-range dependencies to be captured \citep{vaswani2017attention}, are essential for the reasoning/chain-of-thought paradigm \citep{xiong2024effective,guo2025deepseek,jaech2024openai}, and empirically obey scaling laws relating context size to predictive loss \citep{kaplan2020scaling,hoffmann2022training}. They are also expensive: the quadratic cost of attention in the context length has driven substantial research into scaling contexts up \citep{dao2022flashattention,liu2025deepseek,katharopoulos2020transformers,wang2020linformer}. Despite this centrality, a sharp theoretical understanding of \emph{what context length buys} --- and \emph{when reasoning can substitute for it} --- has remained elusive. A core obstacle is the difficulty of formulating quantitative distributional questions about generated text when the underlying language distribution is unknown.

\paragraph{A hierarchical language with tractable statistics.} In this work we study a class of synthetic languages for which the underlying distribution is both known exactly and admits a clean asymptotic analysis. The language is generated by a \emph{broadcast process on a $d$-ary tree}: a root token is sampled from a prior and then propagated downward through a noisy channel until it reaches the $d^h$ leaves, which form the observed sequence. Variants of this model have a long history in probability and statistical physics, and have recently been used as a substrate for theoretical questions about deep learning \citep{mossel2016deep,cagnetta2024compositional,tomasini2024sparse,ren2026provable}. Crucially, the model is not merely a ``formal language'' (i.e., a set of valid strings) but a probability distribution over strings, which is what generative pretraining actually targets via maximum likelihood. The hierarchical structure also reflects a long tradition in linguistics and NLP that treats language as compositional and multiscale \citep{chomsky1957syntactic,jurafsky_martin,elman1990finding,ebeling1995long,yang2016hierarchical}: parse trees, semantic dependencies, and discourse structure all impose correlations across widely separated length scales.

% Motivation for the language model:
% \begin{itemize}
% \item Hierarchy. Researchers in NLP, linguistics, and other areas have constantly emphasized the hierarchical and compositional structure of language \cite{jurafsky_martin,ebeling1995long,chomsky1957syntactic,yang2016hierarchical,elman1990finding,lin2017critical}.
We emphasize that tree-structured representations of language are ubiquitous. Hierarchical sentence diagrams appear in 19th-century pedagogical grammars \citep{reed1877higher},  and \emph{immediate constituent analysis} \citep{bloomfield1933language,wells1947immediate} formalized the view that sentences decompose recursively into nested constituents. \cite{chomsky1957syntactic} and \cite{tesniere1959elements} both represented sentences using hierarchical tree structures, and
%\citet{chomsky1957syntactic} studied phrase-structure grammars, which represent sentences as parse trees built recursively from smaller constituents, and %dependency-based formalisms
%\cite{tesniere1959elements} similarly organized a sentence around a hierarchical structure of head-dependent relations.
compositional semantics assigns meaning to a sentence by recursively combining the meanings of its subtrees \citep{montague1970universal}. NLP inherited this perspective, from probabilistic context-free grammars \citep{jurafsky_martin} and treebank-based parsing \citep{marcus1993building} to tree-structured neural networks \citep{socher2013recursive,tai2015improved}. The broadcast model can be viewed as a probabilistic generative process on such a tree, with the latent internal nodes playing the role of unobserved syntactic or semantic structure.% that constrains the observed leaves.
%Parse trees
% \item Power lows for words (leaf of sub trees) occurrences 
% \item Interesting power low for correlation. \TODO?
% %See e,g, commented paper: 
% %Lin, H., & Tegmark, M. (2017). Critical behavior in physics and probabilistic formal languages. Entropy, 19(7), 299
% \end{itemize}
%Important implementation details:
%Punctuation is probably important in experiments. 
%Should have l punctuation signs for the l levels of the tree.

\paragraph{The broadcast model on trees.}  Let $T_{d,h}$ denote the $d$-ary tree
of height $h$, and let $\Sigma$ be a finite set of tokens. The
$(d, h, \kappa, \nu)$-broadcast process samples a token at each node of the
tree as follows: the root is drawn from a prior $\nu$ over $\Sigma$, and each
non-root node, given the value at its parent, is sampled independently
according to a transition kernel $\kappa$ on $\Sigma$. The \emph{language} is
defined as the joint distribution over the $d^h$ leaf tokens, viewed as a
distribution on $\Sigma^{d^h}$. See \Cref{sec:broadcast-hlm} for full details.

In this paper we focus on two instantiations of this model:
%that exhibit qualitatively different behavior: 
the \emph{Ising} broadcast process,
where $\Sigma = \{\pm 1\}$ and $\kappa$ copies the parent with probability $\rho$
and re-randomizes otherwise, modeling soft global correlations; and the
\emph{coloring} broadcast process, where $\Sigma = [q]$ and $\kappa$ samples a uniformly random color for each child which differs from its parent. The latter models hard logical constraints
analogous to those in code or formal mathematics. In either case, internal
nodes can be interpreted as latent high-level topics or syntactic structure
that constrains the observed tokens at the leaves. 

For our experiments, the leaf sequence is tokenized with \emph{hierarchical
punctuation marks} delimiting subtrees at each level --- analogous to how
natural language uses commas, periods, and paragraph breaks to mark structure
at different scales --- so that a bounded-context model can locate itself
within the hierarchy. See \Cref{sec:simulations} for details. Our theoretical results, which concern the marginal
distribution over leaf values, are unaffected by this tokenization choice.

\paragraph{Our analysis: quantitative predictions for trained transformers via a $k$-gram ansatz.} Our key technical idea is to analyze, in place of a transformer with context length $k$, the optimal autoregressive process that depends only on the previous $k$ tokens. We refer to this as the \emph{$k$-gram ansatz}. While $k$-gram models with large $k$ are well known to be statistically and computationally intractable in general \citep{shannon1948mathematical,jurafsky_martin}, in our hierarchical model the distribution of $k$-gram is tractable, and we use it to derive explicit scaling laws for distributional statistics of the generated sequence. We then show experimentally that transformers trained on the same synthetic language closely match these predictions: the variance, kurtosis, and validity-rate curves of the trained models track our asymptotic theory across a wide range of context sizes. %This gives us evidence that the $k$-gram ansatz is the right object to study if one wants to understand the distributional behavior of bounded-context transformers for this model. 

\subsection{Summary of our results}
We now describe our main theoretical contributions, which fall into two parts: lower bounds showing that bounded-context autoregression is observably inconsistent with the true language, and an upper bound showing that logarithmic-memory reasoning suffices to generate from it exactly.

\begin{figure}[t]
     \centering
     \begin{subfigure}[b]{0.49\textwidth}
         \centering
         \includegraphics[width=\textwidth,trim=20 3 23 3, clip]{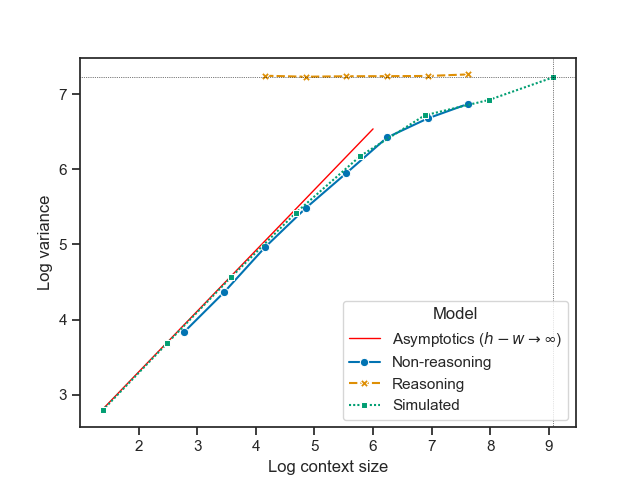}
         \caption{$\log(\frac{1}{n}\Var(\sum X_r))$ versus the log context size.}
         \label{fig:variance-plot}
     \end{subfigure}
     \hfill
     \begin{subfigure}[b]{0.49\textwidth}
         \centering
         \includegraphics[width=\textwidth,trim=3 3 40 3, clip]{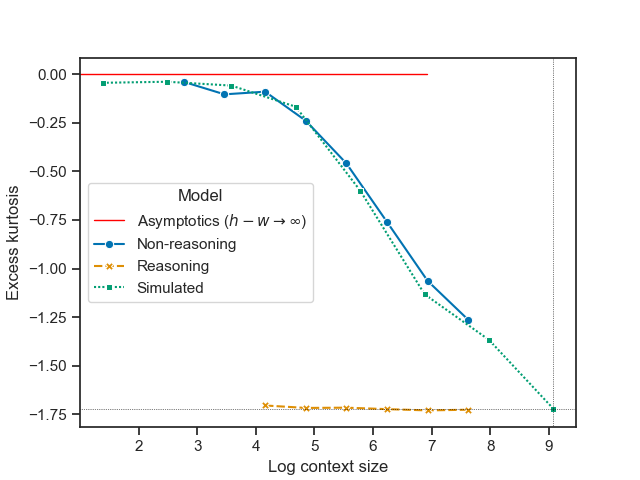}
         \caption{Kurtosis of $\sum X_r$ versus the log context size.}
         \label{fig:kurtosis-plot}
     \end{subfigure}
        \caption{Variance and kurtosis of the sum of leaf spins $\sum X_r$ for the Ising broadcast process with $d=3$, $h=8$, and $\rho=0.9$, computed for different models and context sizes. The thin horizontal lines (grey) indicate the log variance and kurtosis of the ground truth language, and the thin vertical line indicates the context size which encompasses the entire tree (i.e., $w=h$). The ``Non-reasoning'' (blue) and ``Reasoning'' (orange) curves are both trained transformers: the latter model uses additional intermediate tokens in its generation.
        %the former is a simple non-reasoning autoregressive model, where the latter is a reasoning model with workable memory. 
        The ``Simulated'' (green) curves are generated using the exact autoregressive broadcast process in \Cref{def:ar-broadcast}. The ``Asymptotics'' (red) lines are computed from \Cref{thm:log-linear-scaling,thm:kurtosis-main}, which hold when the context depth and the tree height have large gap ($h - w \to \infty$), with $w \to \infty$. See also \Cref{sec:simulation-ising}.}
        \label{fig:ising-plots}
\end{figure}

% \textbf{Autoregressive generation with bounded context.} Our first set of theoretical results precisely quantifies the behavior of \emph{autoregressive models with sublinear context lengths}. To do this, we first introduce an autoregressive variant of the hierarchical language model where the next $d^w$ tokens are sampled in each step using \emph{only} the previous $d^w$ sampled tokens as the context. The final sample of length $d^{h}$ is then the concatenation of the autoregressively sampled tokens: see \Cref{def:ar-broadcast} for the precise definition. A natural question becomes: how do important global statistics of the language scale with the context length?  

\paragraph{Lower bounds: scaling laws and Gaussianity for Ising.} Our lower
bounds concern an autoregressive variant of the broadcast process
(\Cref{def:ar-broadcast}) in which each subtree of size $d^w$ is generated
using only the previous $d^w$ tokens. For the Ising broadcast process above
the Kesten--Stigum threshold $d\rho^2 > 1$, we prove that the normalized
variance of the generated sum of tokens scales log-linearly in the context
depth $w$, and that its kurtosis converges to that of a Gaussian
(\Cref{thm:log-linear-scaling,thm:kurtosis-main}). Both effects imply that
the global coherence of the generated language is polynomially smaller than
that of the true language unless $w$ is within $O(1)$ of $h$, and both are
matched quantitatively by transformers trained on the synthetic language
(\Cref{fig:ising-plots}).

\begin{figure}[t]
    \centering
    \includegraphics[width=0.6\linewidth,trim=10 3 10 3, clip]{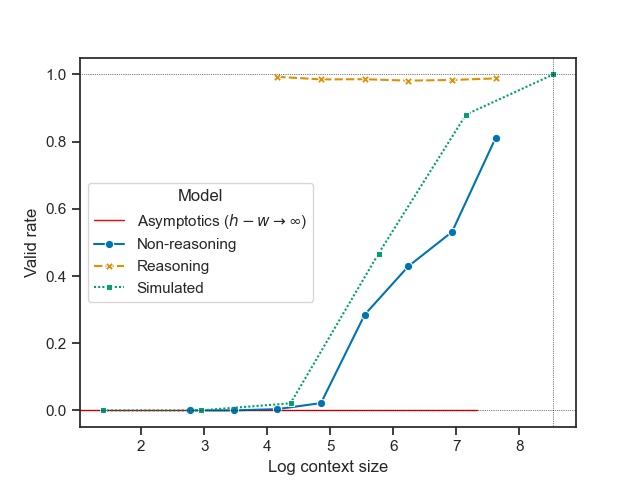}
    \caption{The plot of ``valid rate'' versus log context size for $d=4$, $h=6$, and $q=3$. The sequence of $d^h$ colors generated by a model is valid if there is a $d$-ary tree with height $h$ with proper $q$-coloring such that the leaves have the colors as given in the sequence. The valid rate is defined by the proportion of valid sequences as we generate the sequence multiple times. The thin horizontal lines indicate perfect generation (rate equals 1) and completely inconsistent generation (rate equals 0). The ``Asymptotics'' line is identical to inconsistent generation as per \Cref{thm:freezing-inconsistent}. The vertical line and the other three models are as discussed in \Cref{fig:ising-plots}. See also \Cref{sec:simulation-coloring}.}
    \label{fig:coloring-plot}
\end{figure}

\paragraph{Lower bounds: inconsistent generation for coloring.} For the
coloring broadcast process in the \emph{freezing regime} --- where the
branching factor $d$ is sufficiently large compared to the number of colors
$q$ --- we prove that any bounded-context autoregressive model produces,
with high probability, sequences that are inconsistent with \emph{any}
proper coloring of $T_{d,h}$ (\Cref{thm:freezing-inconsistent}). This is a
qualitatively different failure mode from the Ising case: not statistical
decoherence but outright invalidity, analogous to code in which individual
functions compile but the program as a whole does not. We confirm this
prediction empirically (\Cref{fig:coloring-plot}).

% In this language model, the tokens in each subtree must satisfy the logical constraints induced by being a proper coloring. An important phenomenon in the theory of broadcast processes is the onset of \emph{freezing}: when the branching factor $d$ is sufficiently large compared to the number of tokens $q$, then it turns out that the lower-level tokens are quite likely to \emph{uniquely determine} the higher-level tokens, or subject, in the language. We prove a more striking consequence of \emph{any} bounded context length: we show in \Cref{thm:freezing-inconsistent} that freezing actually leads to \emph{invalid sequences} for any autoregressive model for coloring with high probability: there will likely be \emph{no} proper coloring of the full tree that is consistent with the sequence of tokens at the leaves. In the context of coding, this may correspond to individual functions compiling properly, but not the overall program due to global incompatibility of individual functions. Again, we empirically verify this prediction for pre-trained bounded context models by experimentally computing the fraction of valid generated samples as the context depth grows, see \Cref{fig:coloring-plot}.
\paragraph{Upper bound: logarithmic reasoning suffices.} In contrast to
these lower bounds, we prove that an autoregressive \emph{reasoning} model
--- one equipped with a working memory of $O(h \log d) = O(\log n)$ bits
that can be read from and written to during generation --- can sample
exactly from the true $(d, h, \kappa, \nu)$-language
(\Cref{thm:reasoning}). This is an exponential improvement over what is
achievable without reasoning: our lower bounds show that any non-reasoning
model needs $\Omega(n)$ context to even approximately match simple global
statistics, while $\Theta(\log n)$ reasoning bits suffice to sample exactly.
Trained transformers equipped with reasoning traces match the true language
empirically across all our settings, with context windows orders of magnitude
smaller than the language length.

\subsection{Related Work}
 
\paragraph{Hierarchical language models.} A fundamental question in the theory of deep learning is understanding \emph{why} depth is necessary, that is, identifying natural data distributions for which deep architectures provably outperform shallow ones. \cite{mossel2016deep} proposed to study \emph{hierarchical generative models} based on broadcast on trees with the goal of representing data such as images or languages. 
%In this framework, data is generated top-down through a tree: a class label at the root propagates through layers of latent variables, with noise injected at each level, ultimately producing high-dimensional observations at the leaves. 
Mossel observed that efficient algorithms for reconstructing such models follow from known results on \emph{phylogenetic reconstruction} %Aiming to understand the complexity of inference %and 
%drawing on the classical theory of broadcasting on trees, 
%Mossel 
and established a separation in the semi-supervised setting  {\em below} the Kesten-Stigum bound.
%given access only to low-order moments of the labeled data and all unlabeled data, classification is information-theoretically impossible; yet an efficient iterative algorithm akin to belief propagation, given access to raw labeled and unlabeled samples, can perfectly classify the data with high probability. 
Follow up results established low-degree hardness in the same regime~\citep{KoehlerMossel:22,HuangMossel:24,HuangMossel:25}. 
%The key technical insight is that belief propagation on trees cannot be well approximated by linear functions, which precludes shallow (e.g., kernel-based or moment-based) approaches while permitting deep, layerwise computation.
More recent work in physics~\citep{cagnetta2024compositional,tomasini2024sparse} analyzed variants of this model and demonstrate that both gradient descent and diffusion-based methods experimentally learn their model in the easy regime (i.e. {\em above} the Kesten-Stigum bound).
%Ren et al.~
\cite{ren2026provable} prove that, under mild conditions, a deep convolutional network with gradient descent-based methods and layerwise training can efficiently learn this class. 

\paragraph{Formal languages and the representational power of transformers.} A number of recent works have studied the representational power of the transformer architecture from the perspective of formal languages and worst-case complexity theory. See, e.g.,  \cite{perez2021attention,barcelo2023logical,chiang2023tighter,chen2025theoretical,merrill2023parallelism}. For example, \cite{merrill2023parallelism} showed that log-precision transformers are contained within the class $TC^0$ of bounded-depth circuits, while \cite{perez2021attention,merrill2023expressive,feng2023towards} show that transformers with the ability to generate intermediate states, i.e. with reasoning/CoT, are Turing complete. So from the perspective of worst-case complexity theory (and assuming standard conjectures, c.f., \cite{chen2025theoretical}), reasoning/CoT are known to be crucial to allow transformers to represent many languages. In fact, once CoT is introduced, even relatively simple models like linear next-token-predictors become Turing complete \citep{malach2023auto}.

The aforementioned results are largely \emph{existential}, i.e., investigate the representational ability of transformers. However, existence alone does not imply that standard training methods, or indeed that \emph{any method}, can efficiently learn such a model from examples. Conjecturally hard examples like the problem of learning sparse parities with noise \citep{blum2003noise} show that learning even quite easy-to-represent concepts from data can be cryptographically hard (see related discussion in \cite{malach2023auto}). In contrast, for the hierarchical languages we study, we can experimentally demonstrate that the behavior of realistic transformer architectures, trained via next-token-prediction in the usual way, exhibits close agreement with our theoretical predictions.

\paragraph{Learning-theoretic perspectives on reasoning/CoT.} The recent works of \cite{malach2023auto,joshi2025theory} study the effect of CoT on the computational and statistical tractability of learning. One important finding is that the inclusion of short reasoning traces (of $O(\log n)$ length for an input of length $n$) in the training data can allow autoregressive models to learn functions like parities with noise which are cryptographically hard when the reasoning traces are omitted. The benefit of reasoning in our model is similar, in that a $O(\log n)$ length reasoning trace enables learning both theoretically and experimentally, but the underlying mechanism is \emph{statistical} rather than \emph{computational}. Without the reasoning trace, we show that the transformer with sublinear context windows progressively loses information about the past, which causes it to incorrectly generate the language. This in turn is driven by the \emph{hierarchical} structure of our language, where information naturally flows over different length scales, which is absent from constructions like sparse parities. 

%\paragraph{Context compression.} 
%\paragraph{Beyond finite context length.} \TODO Maybe some related work about how sota LLMs deal with finite context issue?

\paragraph{Other related work.} As discussed in the introduction, the computational challenge caused by the quadratic time-complexity of the attention mechanism, and the importance of being able to attend to large contexts, are considered to be extremely important considerations in the design and use of SoTA language models. Our reasoning example partially fits into the paradigm of \emph{context compression/compaction}. See \cite{anthropic_context} for discussion of the fundamental importance of context management in Claude Code and other applications. 

Our $k$-gram ansatz is partially motivated by a theoretical work of \cite{sharan2018prediction}, which showed some positive results for modeling languages with bounded mutual information over large length scales using simple $k$-gram models. Practical $k$-gram models are restricted to relatively short contexts since their state space grows exponentially with the length of their memory; however, we might hope that transformers with sublinear windows on natural languages may be able to learn good approximations to the corresponding (intractible) $k$-gram model. As a reminder, $k$-gram models themselves date back to the early mathematical studies of natural language by \cite{shannon1948mathematical} --- see, e.g., \cite{jurafsky_martin} for more background.

% \subsection{Our contributions}

% \TODO

\subsection{Organization}
In \Cref{sec:broadcast-hlm} we define the ground truth generative model in full detail. In \Cref{sec:main-theory} we formally state our main theoretical results, leaving the detailed proofs to the appendices. In \Cref{sec:simulations} we confirm our theoretical predictions with experiments on transformers trained on the synthetic language.
%\TODO
 
\section{Broadcast Process as a Hierarchical Language Model}\label{sec:broadcast-hlm}

Formally, our hierarchical model of language follows the ``broadcast process on trees":
\begin{defn}[Broadcast Model]\label{def:broadcast}
    Let $T_{d,h}$ denote the $d$-ary tree of depth $h$ and let $\Sigma$ denote a finite set of tokens. We index the non-root nodes of $T_{d,h}$ by writing a node $r$ at level $\ell$ as $r\in[d]^\ell$.
    Given a probability transition kernel $\kappa$ on $\Sigma$ --- called the \emph{broadcast channel} --- and the initial prior $\nu$ on the root, the \emph{$(d,h,\kappa,\nu)$-broadcast process} on $T_{d,h}$ is generated as follows:
    \begin{enumerate}
        \item Sample the root value $X_{\emptyset}\sim\nu$.
        \item Given the values at a layer $\ell\leq h-1$, independently sample $X_r$ for $r\in[d]^{\ell+1}$ according to the law
        \[
            \mu_r(\sigma)=\kappa(X_{r[1:\ell]},\sigma)\,,\qquad\sigma\in\Sigma\,.
        \]
    \end{enumerate}

    We write $X_L\in\Sigma^{d^h}$ for the sequence of tokens at the leaves. The \emph{$(d,h,\kappa,\nu)$-language} is defined by the law of $X_L$, which is a distribution on $\Sigma^{d^h}$. When $\nu$ is omitted, it is the stationary distribution.
\end{defn} 

%\TODO: we should mention somewhere what it means when $\nu$ is omitted

%Our main contributions are to study how close the output generated by models trained on the $(d,h,\kappa,\nu)$-language is to the ground truth language as a function of context and reasoning length. %for well-studied broadcast channels.
We will study two particular broadcast processes: first, the \emph{Ising broadcast process} is defined by the state space $\Sigma=\{-1,+1\}$ and the broadcast channel
\[
    \kappa(\sigma,\sigma')=\begin{cases}
        \frac{1+\rho}{2}&\text{if $\sigma=\sigma'$,}\\
        \frac{1-\rho}{2}&\text{otherwise}
    \end{cases}
\]
% \[
%     \begin{bmatrix}
%         \frac{1+\rho}{2} & \frac{1-\rho}{2}\\
%         \frac{1-\rho}{2} & \frac{1+\rho}{2}
%     \end{bmatrix}
% \]
for a correlation parameter $\rho\in[0,1]$, and the Rademacher prior $\nu$. In particular, a child copies its parent's state with probability $\rho$ and re-randomizes with probability $1-\rho$.
We also consider the \emph{coloring broadcast process} with state space $\Sigma=[q]$ and broadcast channel
\[
    \kappa(\sigma,\sigma')=\begin{cases}
        0&\text{if $\sigma=\sigma'$,}\\
        \frac{1}{q-1}&\text{otherwise}
    \end{cases}
\]
% \[
%     \begin{bmatrix}
%         0&\frac{1}{q-1}&\cdots&\frac{1}{q-1}\\
%         \frac{1}{q-1}&0&\cdots&\frac{1}{q-1}\\
%         \vdots&\vdots&\ddots&\vdots\\
%         \frac{1}{q-1}&\frac{1}{q-1}&\cdots&0
%     \end{bmatrix}
% \]
for a fixed constant $q\geq2$ representing the number of colors, and the uniform prior $\nu$ on $\Sigma$. Here, a child randomly picks a color not chosen by its parent. This process was studied extensively in probability and statistical physics starting 
with~\cite{KestenStigum:66,Higuchi:77,Spitzer:75}, further background can be found in~\cite{evans2000broadcasting,Mossel:04,MezardMontanari:06,Mossel:22}.

\iffalse
\begin{rmk}
We note that the language generated by the coloring broadcast process is \emph{hard constrained}, in the sense that there is a sequence of $d^h$ colors which is not a valid language, i.e., the set of valid languages is a proper subset of $\Sigma^{d^h}$. On the other hand, any sequence of $d^h$ spins has positive probability under the Ising broadcast process, although typical sequences share various properties (e.g., having an extremely long run of $+1$s in a row is unlikely unless $\rho$ is near $1$).
\end{rmk}
\fi

Our main theoretical results will compare the statistics of the broadcast process with a \emph{Markovian} version that has sublinear length to generate new subtrees:

\begin{defn}[Autoregressive Broadcast Process]\label{def:ar-broadcast}
Given parameters $d,h,\kappa$ as in \Cref{def:broadcast}, the \emph{$(d,h,\kappa)$-autoregressive broadcast process} with context depth $w\leq h$ is the stochastic process $X^L\in \Sigma^{d^h}$ generated as follows:
\begin{enumerate}
    \item Sample $Y_1\in \Sigma^{d^w}$ from the marginal of the $(d,h,\kappa)$-broadcast process on (any) subtree of size $d^w$.

    \item For $r=2,\ldots,d^{h-w}$, sample $Y_r$ by sampling first $h_r\in \{1,\ldots,h-w\}$ independently from the distribution of heights between subtrees of depth $w$ in $T_{d,h}$ and then sampling $Y_r$ from the $(d,h,\kappa)$-broadcast process conditioned on $Y_{r-1}$ and with these subtrees having least common ancestor in $T_{d,h}$ at height exactly $h_r$.
\end{enumerate}
The final sample is $X_L = (Y_1,\ldots,Y_{d^{h-w}})\in \Sigma^{d^h}$.
\end{defn}

The autoregressive broadcast produces a sample $X_L$ by sampling on subtrees of $T_{d,h}$ of size $d^w$ one at a time. In each generation step, the autoregressive process conditions just on the previous subtree to produce the next subtree. The conditional distribution of this next subtree is simulated from the original broadcast process by averaging over the conditional distribution of adjacent subtrees in the original process, taken over the random height of the least common ancestor which governs the signal between these subtrees. In other words, the generation step samples adjacent subtrees according to the law of the broadcast process for a \emph{random} pair of adjacent subtrees in $T_{d,h}$. 

Note that this process is slightly different from the vanilla token-by-token auto-regressive process, because it is more mathematically elegant to analyze a process where you generate one depth $w$ sub-tree at a time. We expect the same results hold for generating one token at a time (which is how we run the experiments), at the cost of a more complicated proof.

\section{Theoretical results}\label{sec:main-theory}
Our main technical results concern the power of autoregressive models where they are only allowed to regress on the limited number of the most recent outputs. 
We consider models that, given an empty prompt or the leaves of a $d$-ary tree with height $w$, output the leaves of a (random) $d$-ary tree with height $w$, where $w$ is the \emph{context depth} as in \Cref{def:ar-broadcast}. This 
%can be formally defined as a 
corresponds to a transition kernel (conditional distribution function)
\[
    p(\cdot\mid -):\Sigma^{d^w}\cup\{\emptyset\}\to\mathcal{P}(\Sigma^{d^w})
\]
where $\mathcal{P}(\Sigma^{d^w})$ is the set of probability measures on $\Sigma^{d^w}$; this is equivalent to considering models with $\Theta(d^w\log|\Sigma|)$ bits of context. Running this model $d^{h-w}$ times autoregressively, starting with the initial prompt $\emptyset$, samples a sequence in $\Sigma^{d^h}$, and our goal is to compare the distribution of this generated sequence to the ground truth $(d,h,\kappa)$-language. 

For simple autoregressive models without reasoning, we consider that the generative model is trained on data generated in the following way:
\begin{enumerate}
    \item Sample a $(d,h,\kappa)$-language $X_L$ of $T_{d,h}$.
    \item Sample a subtree index $i$ uniformly at random from $[d^{h-w}]=\{1,2,\cdots,d^{h-w}\}$.
    \item If $i\geq2$, then the input is the leaves $Y_{i-1}$ of the $(i-1)$th subtree of $T_{d,h}$, and the output is the leaves $Y_{i}$ of the $i$th subtree of $T_{d,h}$. If $i=1$, then the input is an empty prompt.
\end{enumerate}
Since the model can only look at the subtree of height $w$ each step, this is equivalent to training on the $(d,h,\kappa)$-autoregressive broadcast process, and the optimal model will exactly generate the $(d,h,\kappa)$-autoregressive broadcast process. \emph{Thus, it suffices to compare the full broadcast process (\Cref{def:broadcast}) to its Markovian autoregressive variant (\Cref{def:ar-broadcast})}.
%\TODO: it might be ideal to incorporate this argument into the theorems below, but not sure what's the better way.

A natural question is how the global statistics of the autoregressively generated
sequence scale with the context length $w$ as the language length $d^h$ grows.
This section answers this question for the Ising and coloring
broadcast processes (\Cref{sec:ising-theory,sec:coloring-theory}), and shows
that an exponentially smaller amount of \emph{reasoning} memory suffices to
recover the true distribution exactly (\Cref{sec:reasoning-theory}).

\paragraph{Appendices.} Proofs of theorems for \Cref{sec:ising-theory} are deferred to \Cref{sec:moment-autoregressive}, for \Cref{sec:coloring-theory} to \Cref{sec:freezing}, and for \Cref{sec:reasoning-theory} to \Cref{sec:reasoning-proof}.

\subsection{Ising Broadcast Process}\label{sec:ising-theory}

For the Ising broadcast process, a natural statistic is the variance
of the sum of tokens normalized by $d^{h/2}$, which measures the global alignment of the language.
For the true (non-autoregressive) language, the sum of tokens is well-understood:
above the Kesten--Stigum threshold $d\rho^2 > 1$, this statistic becomes correlated
with the latent root token even as $h \to \infty$ \citep{kesten-stigum}, see also~\cite{JansonMossel:04}, and
the normalized variance becomes exponentially large in the height $h$ (see \eqref{eq:true-variance}). The following
theorem shows that the autoregressive variant exhibits a sharply different
behavior in this regime.

\begin{thm}
\label{thm:log-linear-scaling}
Consider the Ising broadcast process on $T_{d,h}$ with $d\rho^2>1$. For a given context depth $0\leq w\leq h$, let $X_L\in \{\pm 1\}^{d^h}$ denote a sample from the autoregressive process on the leaves with context length $d^w$. Then if $w \to \infty$ and $h-w\to \infty$, the \emph{(log-) normalized variance} satisfies:
    \begin{equation}
    \label{eq:normalized-variance}
        \log\left(d^{-h}\cdot\mathsf{Var}\left(\sum_{i=1}^{d^h} (X_L)_i\right)\right) =  w\log(d\rho^2)+\log(A_{d,\rho}(2))+o(1),
    \end{equation}
where $A_{d,\rho}(2)$ is an explicit constant depending only on $d,\rho$.

By contrast, when $X_L\in \{\pm 1\}^{d^h}$ is a sample from the \emph{true} Ising broadcast process (i.e. $w=h$), 
\begin{equation}\label{eq:true-variance}
    \log\left(d^{-h}\cdot\mathsf{Var}\left(\sum_{i=1}^{d^h} (X_L)_i\right)\right) =  h\log(d\rho^2)+\log(C_{d,\rho}(2))+o(1),
\end{equation}
where $C_{d,\rho}(2)\leq A_{d,\rho}(2)$ is an explicit constant depending only on $d,\rho$.\footnote{The different constant terms reflect that when $w$ is very close to $h$, there are finite-size corrections to the asymptotics.}
\end{thm} 

\Cref{thm:log-linear-scaling} has two implications. First, while the
autoregressive process produces tokens with the correct \emph{local} marginals
by definition, its \emph{global} coherence is polynomially smaller than the
true language's unless a linear fraction of the language is used as context
($w = h - O(1)$). In the context of natural language, this can be thought of
as failing to produce a strong unified signal --- for example, generating a
long, ambiguous answer to a yes/no question rather than committing to a
direction. Second, the prediction is precisely \emph{quantitative}: the
log-variance is linear in $w$ with slope $\log(d\rho^2)$, which our experiments
on trained transformers confirm closely (\Cref{fig:variance-plot}).

This decoherence has a further observable consequence. By computing higher
moments of the autoregressive process via the theory of broadcast processes,
we obtain the following result:

\begin{thm}
\label{thm:kurtosis-main}
    Consider the Ising broadcast process on $T_{d,h}$ with $d\rho^2>1$. For a given context depth $0\leq w\leq h$, let $X_L\in \{\pm 1\}^{d^h}$ denote a sample from the autoregressive process on the leaves with context length $d^w$. If $w \to \infty$ and $h-w\to \infty$, then the excess kurtosis of $\sum_i (X_L)_i$ satisfies:
    \begin{equation}
    \label{eq:kurtosis}
    \Kurt\left(\sum_{i=1}^{d^h} (X_L)_i\right)-3:=\frac{\mathbb{E}\left[\left(\sum_{i=1}^{d^h} (X_L)_i\right)^4\right]}{\mathbb{E}\left[\left(\sum_{i=1}^{d^h} (X_L)_i\right)^2\right]^2}-3=o(1)\,.
    \end{equation}
    % \label{eq:kurtosis}
    % \Kurt\left(\sum_{i=1}^{d^h} (X_L)_i\right)-3:=\mathbb{E}\left[\left(\sum_{i=1}^{d^h} (X_L)_i\right)^2\right]^{-2}\mathbb{E}\left[\left(\sum_{i=1}^{d^h} (X_L)_i\right)^4\right]-3=o(1)\,.
    % \end{equation}
\end{thm}

Recall that the kurtosis of Gaussian random variables is precisely $3$. Intuitively, \Cref{thm:kurtosis-main} says that bounded-context autoregression
in this hierarchical language causes information decay rapid enough for
central-limit-type behavior to emerge: the generated sum behaves like a sum
of approximately independent contributions from each subtree, even though the
true language has heavy-tailed multiscale dependencies.%\footnote{In fact, for a certain range of $w<h$, it is possible to use modern forms of the Berry-Esseen theorem for additive functionals of Markov chains (e.g.~\cite{markov-berry-esseen}) to show that the normalized sum of spins converges to a Gaussian. These results do not appear strong enough to ensure Gaussianity for the entire regime of $w,h-w\to \infty$ from \Cref{thm:kurtosis-main}.} 
Our experiments
again confirm this prediction (\Cref{fig:kurtosis-plot}).

%\TODO

\subsection{Coloring Broadcast Process}\label{sec:coloring-theory}
The previous results quantify the behavior of bounded-context autoregression
in a \emph{soft-constrained} language: every sequence has positive probability
under the true language, and the generated sequences differ from it
statistically rather than logically. Many real applications of language models
involve \emph{hard-constrained} languages, where logical or syntactic rules
exclude entire classes of sequences --- code, formal mathematics, structured
data, and so on. The coloring broadcast process serves as a clean abstraction
of this setting: by construction, valid sequences correspond exactly to leaf
labelings extending to a proper coloring of the underlying tree.

A central phenomenon in the theory of colorings on trees
%broadcast processes for coloring 
is
\emph{freezing} (e.g.~\cite{mossel-peres,semerjian,sly-colorings}): when $d$ is sufficiently large relative to $q$, the leaves
of a subtree contain enough information to typically \emph{fix} all of the internal labels, including the
root. We demonstrate that freezing has stark consequences for autoregressive generation.

\begin{thm}
\label{thm:freezing-inconsistent}
    Consider the autoregressive coloring broadcast process with $q$ colors on $T_{d,h}$ with \emph{any} context depth $w<h$ satisfying $d>q(\log(q)+\log\log(q)+1+o_q(1))$.\;%\footnote{For familiar readers, we remark that this occurs well below the Kesten-Stigum bound for colorings ($d = (q-1)^2$).} 
    Then with probability $1-o_q(1)$, the sampling of the leaves in the autoregressive coloring process is inconsistent with any proper coloring of $T_{d,h}$.
\end{thm}

The interpretation is that with bounded context, the autoregressive process
samples each subtree consistently with itself but cannot maintain global
consistency across subtrees, as freezing-determined ``commitments''
made at the leaves of one subtree generically conflict with those of another.
In the natural-language analogy: each subtree corresponds to a coherent local
fragment, but the overall sequence has no globally valid parse. The empirical
counterpart in \Cref{fig:coloring-plot} shows that trained transformers in
this regime almost never produce a valid sequence, while reasoning-equipped
models almost always do.

% \begin{figure}[t]
%     \centering
%     \includegraphics[width=0.6\linewidth]{figures/validrate.png}
%     \caption{The plot of ``valid rate'' versus the log context size for $d=4$, $h=6$, and $q=3$. The sequence of $d^h$ colors generated by a model is valid if there is a $d$-ary tree with height $h$ with proper $q$-coloring such that the leaves have the colors as given in the sequence. The valid rate is defined by the proportion of valid sequences as we generate the sequence multiple times. The thin horizontal lines indicate perfect generation (rate equals 1) and inconsistent generation (rate equals 0). The vertical line and the three models are as discussed in \Cref{fig:ising-plots}.}
%     \label{fig:coloring-plot}
% \end{figure}

% \begin{thm}
%     Autoregressive models produce junk in frozen regime. \TODO
% \end{thm}

\subsection{Autoregressive Models with Reasoning}\label{sec:reasoning-theory}

% An \emph{autoregressive reasoning model} can be thought of as the ordinary autoregressive model but with an additional working memory, which is not part of the model outputs but which the model can freely read from or write to, possibly in a different language. Given a finite set $\mathcal{M}$ of \emph{memory states}, an autoregressive reasoning model can be formalized as the conditional distribution function
% \begin{equation}\label{eq:reasoning-p}
%     p(\cdot\mid -):(\Sigma^{d^w}\cup\{\emptyset\})\times\mathcal{M}\to\mathcal{P}(\Sigma^{d^w}\times\mathcal{M})\,.
% \end{equation}
% The context size of this reasoning model is $\Theta(d^w+\log|\mathcal{M}|)$. After each autoregressive sampling step, only the $\Sigma^{d^w}$ component of the sampled value is considered the model output. Thus, running the sampling $d^{h-w}$ times samples a sequence in $\Sigma^{d^h}$, just as in the non-reasoning case.

We model a reasoning-equipped autoregressive model as the ordinary
autoregressive model augmented with an additional working memory --- a finite
set $\mathcal{M}$ of memory states that the model may freely read from or
write to between generation steps, and whose contents are not part of the
final output. % (\Cref{eq:reasoning-p}). 
%Given the set $\mathcal{M}$, an autoregressive reasoning model can be formalized 
The model is then a transition kernel (conditional distribution function)
\begin{equation}\label{eq:reasoning-p}
     p(\cdot\mid -):(\Sigma^{d^w}\cup\{\emptyset\})\times\mathcal{M}\to\mathcal{P}(\Sigma^{d^w}\times\mathcal{M})\,.
\end{equation}
The size of the memory $\log|\mathcal{M}|$
is the relevant ``reasoning budget,'' analogous to the length of a
chain-of-thought trace in real LLMs. The next theorem shows that a
logarithmic budget is sufficient to sample the true language exactly, even
with arbitrarily small context.

\begin{thm}\label{thm:reasoning}
    For any $0\leq w\leq h$, $|\mathcal{M}|=(d|\Sigma|)^{O(h-w)}$ suffices to sample from the exact $(d,h,\kappa)$-language. In particular, there is an autoregressive reasoning model with context size $O(h\log (d|\Sigma|))$ that samples from the ground truth language (when $w=0$ is chosen).
\end{thm}
% \begin{proof}
%     See \Cref{sec:reasoning-proof}.
% \end{proof}

This result is agnostic to the broadcast channel $\kappa$.
The intuition is that the memory only needs to carry information about the
\emph{path from the root to the current generation position} in the tree ---
specifically, the labels of the $O(h)$ ancestors and recent siblings needed to
sample the next subtree from the true conditional distribution. This is
$O(h \log d |\Sigma|) = O(\log n)$ bits of state, exponentially smaller than the
$\Omega(n)$ context required by a non-reasoning model. Connecting back to
practice, this offers a clean theoretical analogue to the
\emph{context-compression} phenomenon studied in modern LLMs
\citep{anthropic_context}: a small, well-chosen working memory can substitute
for a much larger context window, provided the underlying
distribution has compressible hierarchical structure.

\section{Empirical Results}\label{sec:simulations}

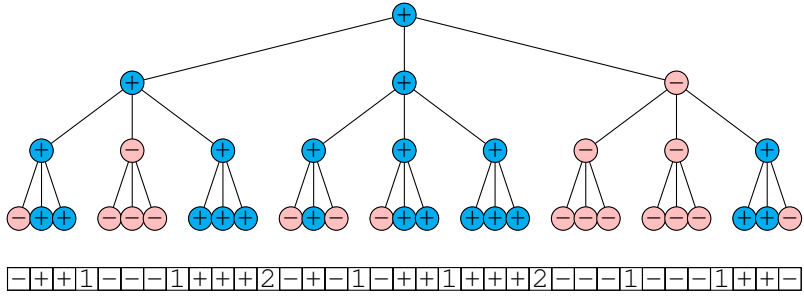
\begin{figure}
    \centering
    \begin{tikzpicture}[
  level distance=9mm,
  tree node/.style={
    circle,
    draw,
    minimum size=3mm,
    inner sep=0pt,
    text height=1.2ex
  },
  sq/.style={
    rectangle,
    draw,
    minimum size=3mm,
    inner sep=0pt,
    text height=1.2ex
  },
  level 1/.style={sibling distance=36mm},
  level 2/.style={sibling distance=12mm},
  level 3/.style={sibling distance=3mm}
]

% --- tree itself ---
\node[tree node] (r) {}
  child foreach \i in {1,2,3} {
    node[tree node] (a-\i) {}
    child foreach \j in {1,2,3} {
      node[tree node] (b-\i-\j) {}
      child foreach \k in {1,2,3} {
        node[tree node] (l-\i-\j-\k) {}
      }
    }
  };

% --- one square below each leaf ---
\foreach \i in {1,2,3}{
  \foreach \j in {1,2,3}{
    \foreach \k in {1,2,3}{
      \node[sq] (s-\i-\j-\k) at ($ (l-\i-\j-\k) + (0,-8mm) $) {};
      % optional vertical edge:
      % \draw (l-\i-\j-\k) -- (s-\i-\j-\k);
    }
  }
}

% --- squares between adjacent leaves with different immediate parents ---
% inside each block of 9 leaves:
\foreach \i in {1,2,3}{
  \foreach \j in {1,2}{
    \pgfmathtruncatemacro{\jp}{\j+1}
    \node[sq] (p-\i-\j) at ($ (l-\i-\j-3)!0.5!(l-\i-\jp-1) + (0,-8mm) $) {\texttt{1}};
  }
}

% across the big blocks (between i and i+1):
\foreach \i in {1,2}{
  \pgfmathtruncatemacro{\ip}{\i+1}
  \node[sq] (q-\i) at ($ (l-\i-3-3)!0.5!(l-\ip-1-1) + (0,-8mm) $) {\texttt{2}};
}

\node[tree node, fill=cyan] at (r) {\texttt{+}};

\node[tree node, fill=cyan] at (a-1) {\texttt{+}};
\node[tree node, fill=cyan] at (a-2) {\texttt{+}};
\node[tree node, fill=pink] at (a-3) {\texttt{-}};

\node[tree node, fill=cyan] at (b-1-1) {\texttt{+}};
\node[tree node, fill=pink] at (b-1-2) {\texttt{-}};
\node[tree node, fill=cyan] at (b-1-3) {\texttt{+}};

\node[tree node, fill=cyan] at (b-2-1) {\texttt{+}};
\node[tree node, fill=cyan] at (b-2-2) {\texttt{+}};
\node[tree node, fill=cyan] at (b-2-3) {\texttt{+}};

\node[tree node, fill=pink] at (b-3-1) {\texttt{-}};
\node[tree node, fill=pink] at (b-3-2) {\texttt{-}};
\node[tree node, fill=cyan] at (b-3-3) {\texttt{+}};

\node[tree node, fill=pink] at (l-1-1-1) {\texttt{-}};
\node[tree node, fill=cyan] at (l-1-1-2) {\texttt{+}};
\node[tree node, fill=cyan] at (l-1-1-3) {\texttt{+}};

\node[tree node, fill=pink] at (l-1-2-1) {\texttt{-}};
\node[tree node, fill=pink] at (l-1-2-2) {\texttt{-}};
\node[tree node, fill=pink] at (l-1-2-3) {\texttt{-}};

\node[tree node, fill=cyan] at (l-1-3-1) {\texttt{+}};
\node[tree node, fill=cyan] at (l-1-3-2) {\texttt{+}};
\node[tree node, fill=cyan] at (l-1-3-3) {\texttt{+}};

\node[tree node, fill=pink] at (l-2-1-1) {\texttt{-}};
\node[tree node, fill=cyan] at (l-2-1-2) {\texttt{+}};
\node[tree node, fill=pink] at (l-2-1-3) {\texttt{-}};

\node[tree node, fill=pink] at (l-2-2-1) {\texttt{-}};
\node[tree node, fill=cyan] at (l-2-2-2) {\texttt{+}};
\node[tree node, fill=cyan] at (l-2-2-3) {\texttt{+}};

\node[tree node, fill=cyan] at (l-2-3-1) {\texttt{+}};
\node[tree node, fill=cyan] at (l-2-3-2) {\texttt{+}};
\node[tree node, fill=cyan] at (l-2-3-3) {\texttt{+}};

\node[tree node, fill=pink] at (l-3-1-1) {\texttt{-}};
\node[tree node, fill=pink] at (l-3-1-2) {\texttt{-}};
\node[tree node, fill=pink] at (l-3-1-3) {\texttt{-}};

\node[tree node, fill=pink] at (l-3-2-1) {\texttt{-}};
\node[tree node, fill=pink] at (l-3-2-2) {\texttt{-}};
\node[tree node, fill=pink] at (l-3-2-3) {\texttt{-}};

\node[tree node, fill=cyan] at (l-3-3-1) {\texttt{+}};
\node[tree node, fill=cyan] at (l-3-3-2) {\texttt{+}};
\node[tree node, fill=pink] at (l-3-3-3) {\texttt{-}};

\node[sq] at (s-1-1-1) {\texttt{-}};
\node[sq] at (s-1-1-2) {\texttt{+}};
\node[sq] at (s-1-1-3) {\texttt{+}};

\node[sq] at (s-1-2-1) {\texttt{-}};
\node[sq] at (s-1-2-2) {\texttt{-}};
\node[sq] at (s-1-2-3) {\texttt{-}};

\node[sq] at (s-1-3-1) {\texttt{+}};
\node[sq] at (s-1-3-2) {\texttt{+}};
\node[sq] at (s-1-3-3) {\texttt{+}};

\node[sq] at (s-2-1-1) {\texttt{-}};
\node[sq] at (s-2-1-2) {\texttt{+}};
\node[sq] at (s-2-1-3) {\texttt{-}};

\node[sq] at (s-2-2-1) {\texttt{-}};
\node[sq] at (s-2-2-2) {\texttt{+}};
\node[sq] at (s-2-2-3) {\texttt{+}};

\node[sq] at (s-2-3-1) {\texttt{+}};
\node[sq] at (s-2-3-2) {\texttt{+}};
\node[sq] at (s-2-3-3) {\texttt{+}};

\node[sq] at (s-3-1-1) {\texttt{-}};
\node[sq] at (s-3-1-2) {\texttt{-}};
\node[sq] at (s-3-1-3) {\texttt{-}};

\node[sq] at (s-3-2-1) {\texttt{-}};
\node[sq] at (s-3-2-2) {\texttt{-}};
\node[sq] at (s-3-2-3) {\texttt{-}};

\node[sq] at (s-3-3-1) {\texttt{+}};
\node[sq] at (s-3-3-2) {\texttt{+}};
\node[sq] at (s-3-3-3) {\texttt{-}};

\end{tikzpicture}
    \caption{An example of the tokenization of the leaves of a ternary tree with height $3$.}
    \label{fig:ternary}
\end{figure}

% \begin{itemize}
%     \item Can transformers learn to predict well even given most of the leaves? Then check the empirical variance of the sum for different contexts lengths. What values of $n$ and $\theta$ to try?

%     \item Can trained models do well when given the $O(\log n)$ memory that we know/hope works well?

%     \item Can trained models learn how to use small working memory to predict well? Not sure how exactly to implement...is RL needed? (cf. Recurrent Memory Transformer, Compressive Transformer, ...)

%     \item Same story for colorings.
    
% \end{itemize}

We corroborate our theory by empirically demonstrating the scaling law for the two broadcast processes: the sum of leaves for the Ising model, and the success rate of root reconstruction for the coloring model. We use \emph{nanochat} by \cite{nanochat} as our base model and train on the dataset synthetically generated from our hierarchical language described in \Cref{sec:broadcast-hlm}. Then we evaluate the empirical distribution of the output generated by the autoregressive sampling of the pre-trained model. Our codes for experiments are available in \url{https://github.com/joonhyungshin/nanocontext}. See \Cref{apdx:experiment} for the detailed experiment setup.

\paragraph{Punctuation.} A notable detail in the tokenization step of our language is the introduction of \emph{punctuation marks}. Recall that our hierarchical language is a sequence of the $d^h$ leaves of a $d$-ary tree with height $h$, with latent internal nodes that affect the leaves in its subtree. We insert a punctuation mark every $d$ values in the sequence, effectively telling the model that ``a certain subtree has concluded, and a new subtree begins.'' We use different punctuation tokens for different height of the subtrees, analogous to natural language documents which have different indicators at different levels (commas, periods, line breaks, etc.). Including the punctuation marks, we use $d^{h-1}(d+1)-1$ many tokens to tokenize a $d$-ary tree with height $h$. See \Cref{fig:ternary} for an illustrative example. %and \Cref{apdx:tokenization} for the mathematical definition of this procedure.

\paragraph{Autoregressive training.} Recall that in our theoretical results in \Cref{sec:main-theory}, we considered restrictive models that samples a subtree given the previous subtree. To evaluate under a more practical scenario, we do not follow this restriction in our experiments. Instead, we do the usual autoregressive training: the training samples are ``chunks'' $C$ of the leaves of length $k$ equal to the prescribed context size, and train the model to predict $C[2:k+1]$ given $C[1:k]$. During the inference, the model autoregressively outputs a token one at a time, where the current output is fed back as the most recent token of the input and the oldest token is truncated in case the number of tokens exceeds $k$. Here, the punctuation marks described in the previous paragraph play a vital role, since it tells the model what part of the subtree it is currently generating or being trained on. 
%See \Cref{apdx:train-vanilla} for further details regarding training data generation.

\paragraph{Reasoning models.} 
%Since we do not train our base models on natural languages, it is not possible to ``prompt'' the model to maintain a working memory. 
%Our approach to
%Instead, 
We train the reasoning model from scratch in a supervised fashion, by manually inserting the desired memory states periodically in the training data. For instance, suppose that we aim at a reasoning model $p(\cdot\mid -)$ as in \Cref{eq:reasoning-p} capable of exact sampling. We generate a sequence
\[
    L_0,M_0,L_1,M_1,\ldots,L_{d^{h-w}},M_{d^{h-w}}
\]
where $(L_{i+1},M_{i+1})$ are sampled from $p(\cdot\mid L_i,M_i)$ autoregressively, tokenize the sequence, randomly sample a chunk of length $k$, and feed it as a training sample. To inform the model wrapper that part of the output is the memory state and not the actual leaves, we prepend and append special memory tokens
%in the tokenization step of $M_i$ 
to mark the beginning and the end of the memory section. %More details can be found in \Cref{apdx:train-reasoning}.

% \paragraph{Computation.} We note that nanochat is powerful but lightweight enough to reach a GPT-2 grade capability for natural language in around 2 hours training on 8xH100 GPU node. Due to the simplicity of our synthetic language, all of the pre-training runs in our experiments take a few hours even with 2 A100 GPUs, which makes them easily reproducible.

%Experiments without reasoning are conducted by simply training on randomly sampled chunks of the language. To equip the language model with a reasoning capability, we insert special tokens that encode the ``summary'' of tokens the model generated so far.

\subsection{Ising Broadcast Process}\label{sec:simulation-ising}

We collect the $d^h$ spins autoregressively outputted by the models and compute their sum $\sum X_r$. We repeat this at least 1,000 times and visualize distributional properties of $\sum X_r$. To compare with the theory results in \Cref{sec:ising-theory}, we estimate the log-normalized variances as predicted in by \Cref{eq:normalized-variance} in \Cref{thm:log-linear-scaling}, %normalized variance
%\[
%    \Var\left(\frac{1}{d^h}\sum_{r\in[d]^h}X_r\right)
%\]
as well as the excess kurtosis as in \Cref{eq:kurtosis} as predicted in \Cref{thm:kurtosis-main}.
% These are computed using the standard sample variance and sample excess kurtosis formulae commonly available in scientific computing packages (see, e.g., \cite{zwillinger1999crc}). Note that normal distributions and the Rademacher distribution have excess kurtosis $0$ and $-2$, respectively.

We run the experiments with correlation $\rho=0.9$ on ternary trees ($d=3$) of height $h=8$, so we are above the Kesten--Stigum threshold ($d\rho^2=2.43>1$). We train reasoning models with context sizes $2^6,2^7,\cdots,2^{11}$ and non-reasoning models with context sizes $2^4,2^5,\cdots,2^{11}$. The results are summarized in \Cref{fig:ising-plots}. The ``Simulated'' curve is computed using the true autoregressive broadcast process, with context depth $w$ defined in \Cref{def:ar-broadcast}. 
%Note that for the ``Simulated'' data points, the context size is defined to be $d^{w-1}(d+1)$, instead of $d^w$, because we take the punctuation marks into account. %apdx
For the non-reasoning models, as context size decreases, the variance and the kurtosis deviate from the ground truth and converge to our asymptotic prediction as $w$ shrinks. In contrast, for reasoning models the variance and kurtosis are close to the ground truth 
%(indicated by the thin black horizontal line), 
even at context size $2^6=64 \ll d^h=6561$.

\subsection{Coloring Broadcast Process}\label{sec:simulation-coloring}

In the $q$-coloring model, we collect the $d^h$ colors from the model output, and attempt to find a $d$-ary tree with height $h$ with proper $q$-coloring such that the colors of the leaves are those generated by the model. We repeat this at least 1000 times and count the number of successful reconstruction attempts to estimate the success rate.

We run the experiments with quaternary trees ($d=4$) of height $h=6$ on $q=3$ colors, lying in the frozen regime. Context sizes are as in the Ising case.
%The reasoning models are trained with context sizes $2^6,\cdots,2^{11}$ and the non-reasoning models are trained with context sizes $2^4,\cdots,2^{11}$, same with the Ising model experiments.
As shown in \Cref{fig:coloring-plot}, non-reasoning models with small context size almost never generate a sequence of valid colors, while the reasoning models almost always generate a valid sequence.

\section{Limitations}
\label{sec:limitations}
%While we have shown 
A key advantage of our hierarchical model of language is to enable clean theory and empirical validation. It would be interesting to study both theoretically and empirically similar phenomena in natural languages.
%this model is of course quite stylized compared to natural language. 
For our theoretical results, we leverage the geometry of regular trees to obtain precise quantitative predictions. Extending this theory to more general language or correlation structures %that abstract other important features of natural language 
is an interesting future direction. %On the experimental side, our results are limited to smaller token spaces ($\vert \Sigma\vert=2,3$ vs. 200K for frontier models (\cite{tiktoken})) and context lengths ($d^h\leq 7000$ vs. 1M for frontier models (\cite{anthropic2026claudeopus47,google2026gemini31pro,openai2026gpt55})). While these sizes are sufficient to agree with our \emph{asymptotic} results, it would be interesting to determine how these conclusions scale.

% \section{Evaluation Metrics}
% \begin{itemize}
% \item TV - probably too strong?
% \item Local KL and Global ?
% \item Wasserstein - what is the right normalization? $1/n$? Average distance given random prompts? 
% \item Locality of context. E.g. define by exponential windows:
% $I(X_0^n, X_{-1})$ then $I(X_0^n, X_{-2}^{-3} | X_{-1})$ etc. 
% \item can also consider the following metric for good generation by a language which is to check the empirical distribution of the posterior at the root vs. the actual posterior. More generally look at the empirical distribution of the top small number of layers vs. the actual. 
% \end{itemize}

% \section{Examples}
% \begin{itemize}
% \item Markov with long memory.
% \item Tree model?
% \end{itemize}

\section*{Acknowledgments}
J.G. and E.M were partially supported by Vannevar Bush Faculty Fellowship ONR-N00014-20-1-2826 and by NSF award NSF DMS-2031883. We thank Yonatan Belinkov, Tatsunori Hashimoto, and Matus Telgarsky for interesting discussions.

\bibliography{refs,my,all}

@article {Higuchi:77,
    AUTHOR = {Y. Higuchi},
     TITLE = {Remarks on the limiting {G}ibbs states on a {$(d+1)$}-tree},
   JOURNAL = {Publ. Res. Inst. Math. Sci.},
    VOLUME = {13},
      YEAR = {1977},
    NUMBER = {2},
     PAGES = {335--348},
}

@article {KestenStigum:66,
    AUTHOR = {H. Kesten and B. P. Stigum},
     TITLE = {Additional limit theorems for indecomposable multidimensional
              {G}alton-{W}atson processes},
   JOURNAL = {Ann. Math. Statist.},
    VOLUME = {37},
      YEAR = {1966},
     PAGES = {1463--1481},
}

@article{MezardMontanari:06,
  author={M. M\'{e}zard and A. Montanari},
  title={Reconstruction on Trees and the Spin Glass Transition},
 journal={Journal of Statistical Physics},
 year={2006},
 volume={124},
 pages={1317--1350}
}

@article {Spitzer:75,
    AUTHOR = {F. Spitzer},
     TITLE = {Markov random fields on an infinite tree},
   JOURNAL = {Ann. Probability},
    VOLUME = {3},
      YEAR = {1975},
    NUMBER = {3},
     PAGES = {387--398},
   MRCLASS = {60K35},
}

@INCOLLECTION{Mossel:04,
    AUTHOR = {E. Mossel},
    AUTHOR_AR = {E. Mossel},
    AUTHOR_HTML_MR = {Mossel, E.},
    AUTHOR_ID_MR = {637297},
    AUTHOR_MR = {Mossel, E.},
    TITLE = {Survey: Information flow on trees},
    TITLE_AR = {Survey: Information flow on trees},
    TITLE_HT = {Survey: Information flow on trees},
    HOWPUBLISHED_AR = {},
        ID_AR = {math.PR/040644},
    BOOKTITLE = {Graphs, Morphisms and Statistical Physics. DIMACS series in discrete mathematics and theoretical computer science},
    EDITOR = {J. Ne\v{s}et\v{r}il and P. Winkler},
    URL = {http://front.math.ucdavis.edu/0406.5446},
    YEAR = {2004},
      AUTHOR_MR = {Mossel, E.},
     TITLE_MR = {Survey: information flow on trees},
 BOOKTITLE_MR = {Graphs, morphisms and statistical physics},
    SERIES_MR = {DIMACS Ser. Discrete Math. Theoret. Comput. Sci.},
    VOLUME_MR = {63},
     PAGES_MR = {155--170},
     PAGES    = {155--170},
 PUBLISHER_MR = {Amer. Math. Soc.},
   ADDRESS_MR = {Providence, RI},
      YEAR_MR = {2004},
  MRNUMBER = {MR2056226},
    MSC_MR = {60J80 (60F05 60K35 82C20 94A15)},
        PAGES_MR={155--170},
        PUBLISHER_MR={American Mathematical Society},
        ID_MR={2056226},
    catId={papers},
    SUBJECTS={prob bio alg},
}

@ARTICLE{JansonMossel:04,
    AUTHOR = {S. Janson and E. Mossel},
    AUTHOR_AR = {Svante Janson and E. Mossel},
    TITLE = {Robust reconstruction on trees is determined by the second eigenvalue},
    TITLE_AR = {Robust reconstruction on trees is determined by the second eigenvalue},
        ID_AR = {math.PR/0406447},
        YEAR = {2004},
    URL = {http://front.math.ucdavis.edu/0406.5447},
        JOURNAL = {Ann. Probab.},
        VOLUME = {32},
        PAGES = {2630--2649},
        ID_MR={2078553},
    catId={papers},
    SUBJECTS={prob},
}

@article{Mossel:22,
  title={Probabilistic view of voting, paradoxes, and manipulation},
  author={Mossel, Elchanan},
  journal={BULLETIN OF THE AMERICAN MATHEMATICAL SOCIETY},
  volume={59},
  number={3},
  pages={297--330},
  year={2022}
}

@article{mossel2016deep,
  title={Deep learning and hierarchal generative models},
  author={Mossel, Elchanan},
  journal={arXiv preprint arXiv:1612.09057},
  year={2016}
}

@inproceedings{HuangMossel:24,
title={Low Degree Hardness for Broadcasting on Trees},
author={Han Huang and Elchanan Mossel},
booktitle={The Thirty-eighth Annual Conference on Neural Information Processing Systems},
year={2024},
pages={32137--32218},
url={https://openreview.net/forum?id=3iOefhez5e}
}

@InProceedings{HuangMossel:25,
  title = 	 {Polynomial low degree hardness for Broadcasting on Trees (Extended Abstract)},
  author =       {Huang, Han and Mossel, Elchanan},
  booktitle = 	 {Proceedings of Thirty Eighth Conference on Learning Theory},
  pages = 	 {2856--2857},
  year = 	 {2025},
  editor = 	 {Haghtalab, Nika and Moitra, Ankur},
  volume = 	 {291},
  series = 	 {Proceedings of Machine Learning Research},
  month = 	 {30 Jun--04 Jul},
  publisher =    {PMLR},
  url = 	 {https://proceedings.mlr.press/v291/huang25a.html},
}

@article{KoehlerMossel:22,
title={Reconstruction on Trees and Low-Degree Polynomials},
author={Frederic Koehler and Elchanan Mossel},
  journal={Advances in Neural Information Processing Systems},
  volume={35},
  pages={18942--18954},
  year={2022}
}

@article{cagnetta2024compositional,
  title={How deep neural networks learn compositional data: The random hierarchy model},
  author={Cagnetta, Francesco and Petrini, Leonardo and Tomasini, Umberto M and Favero, Alessandro and Wyart, Matthieu},
  journal={Physical Review X},
  volume={14},
  number={3},
  pages={031001},
  year={2024},
  publisher={APS}
}

@article{tomasini2024sparse,
  title={How deep networks learn sparse and hierarchical data: the sparse random hierarchy model},
  author={Tomasini, Umberto and Wyart, Matthieu},
  journal={arXiv preprint arXiv:2404.10727},
  year={2024}
}

@article{ren2026provable,
  title={Provable Learning of Random Hierarchy Models and Hierarchical Shallow-to-Deep Chaining},
  author={Ren, Yunwei and Dandi, Yatin and Krzakala, Florent and Lee, Jason D},
  journal={arXiv preprint arXiv:2601.19756},
  year={2026}
}

@misc{nanochat,
  author = {Andrej Karpathy},
  title = {nanochat: The best ChatGPT that \$100 can buy},
  year = {2025},
  publisher = {GitHub},
  url = {https://github.com/karpathy/nanochat}
}

@article{barcelo2023logical,
  title={Logical languages accepted by transformer encoders with hard attention},
  author={Barcel{\'o}, Pablo and Kozachinskiy, Alexander and Lin, Anthony Widjaja and Podolskii, Vladimir},
  journal={arXiv preprint arXiv:2310.03817},
  year={2023}
}

@inproceedings{chiang2023tighter,
  title={Tighter bounds on the expressivity of transformer encoders},
  author={Chiang, David and Cholak, Peter and Pillay, Anand},
  booktitle={International Conference on Machine Learning},
  pages={5544--5562},
  year={2023},
  organization={PMLR}
}

@article{perez2021attention,
  title={Attention is turing-complete},
  author={P{\'e}rez, Jorge and Barcel{\'o}, Pablo and Marinkovic, Javier},
  journal={Journal of Machine Learning Research},
  volume={22},
  number={75},
  pages={1--35},
  year={2021}
}

@article{malach2023auto,
  title={Auto-regressive next-token predictors are universal learners},
  author={Malach, Eran},
  journal={arXiv preprint arXiv:2309.06979},
  year={2023}
}

@article{joshi2025theory,
  title={A theory of learning with autoregressive chain of thought},
  author={Joshi, Nirmit and Vardi, Gal and Block, Adam and Goel, Surbhi and Li, Zhiyuan and Misiakiewicz, Theodor and Srebro, Nathan},
  journal={arXiv preprint arXiv:2503.07932},
  year={2025}
}

@inproceedings{chen2025theoretical,
  title={Theoretical limitations of multi-layer transformer},
  author={Chen, Lijie and Peng, Binghui and Wu, Hongxun},
  booktitle={2025 IEEE 66th Annual Symposium on Foundations of Computer Science (FOCS)},
  pages={2631--2653},
  year={2025},
  organization={IEEE}
}

@article{merrill2023parallelism,
  title={The parallelism tradeoff: Limitations of log-precision transformers},
  author={Merrill, William and Sabharwal, Ashish},
  journal={Transactions of the Association for Computational Linguistics},
  volume={11},
  pages={531--545},
  year={2023}
}

@article{merrill2023expressive,
  title={The expressive power of transformers with chain of thought},
  author={Merrill, William and Sabharwal, Ashish},
  journal={arXiv preprint arXiv:2310.07923},
  year={2023}
}

@article{feng2023towards,
  title={Towards revealing the mystery behind chain of thought: a theoretical perspective},
  author={Feng, Guhao and Zhang, Bohang and Gu, Yuntian and Ye, Haotian and He, Di and Wang, Liwei},
  journal={Advances in Neural Information Processing Systems},
  volume={36},
  pages={70757--70798},
  year={2023}
}

@article{blum2003noise,
  title={Noise-tolerant learning, the parity problem, and the statistical query model},
  author={Blum, Avrim and Kalai, Adam and Wasserman, Hal},
  journal={Journal of the ACM (JACM)},
  volume={50},
  number={4},
  pages={506--519},
  year={2003},
  publisher={ACM New York, NY, USA}
}

@online{anthropic_context,
title = {Effective context engineering for AI agents},
author = {Anthropic},
year = {2025},
url = {https://www.anthropic.com/engineering/effective-context-engineering-for-ai-agents}
}

@inproceedings{yang2016hierarchical,
  title={Hierarchical attention networks for document classification},
  author={Yang, Zichao and Yang, Diyi and Dyer, Chris and He, Xiaodong and Smola, Alex and Hovy, Eduard},
  booktitle={Proceedings of the 2016 conference of the North American chapter of the association for computational linguistics: human language technologies},
  pages={1480--1489},
  year={2016}
}

@article{vaswani2017attention,
  title={Attention is all you need},
  author={Vaswani, Ashish and Shazeer, Noam and Parmar, Niki and Uszkoreit, Jakob and Jones, Llion and Gomez, Aidan N and Kaiser, {\L}ukasz and Polosukhin, Illia},
  journal={Advances in neural information processing systems},
  volume={30},
  year={2017}
}

@article{dao2022flashattention,
  title={Flashattention: Fast and memory-efficient exact attention with io-awareness},
  author={Dao, Tri and Fu, Dan and Ermon, Stefano and Rudra, Atri and R{\'e}, Christopher},
  journal={Advances in neural information processing systems},
  volume={35},
  pages={16344--16359},
  year={2022}
}

@article{liu2025deepseek,
  title={Deepseek-v3. 2: Pushing the frontier of open large language models},
  author={Liu, Aixin and Mei, Aoxue and Lin, Bangcai and Xue, Bing and Wang, Bingxuan and Xu, Bingzheng and Wu, Bochao and Zhang, Bowei and Lin, Chaofan and Dong, Chen and others},
  journal={arXiv preprint arXiv:2512.02556},
  year={2025}
}

@inproceedings{katharopoulos2020transformers,
  title={Transformers are rnns: Fast autoregressive transformers with linear attention},
  author={Katharopoulos, Angelos and Vyas, Apoorv and Pappas, Nikolaos and Fleuret, Fran{\c{c}}ois},
  booktitle={International conference on machine learning},
  pages={5156--5165},
  year={2020},
  organization={PMLR}
}

@article{wang2020linformer,
  title={Linformer: Self-attention with linear complexity},
  author={Wang, Sinong and Li, Belinda Z and Khabsa, Madian and Fang, Han and Ma, Hao},
  journal={arXiv preprint arXiv:2006.04768},
  year={2020}
}

@article{kaplan2020scaling,
  title={Scaling laws for neural language models},
  author={Kaplan, Jared and McCandlish, Sam and Henighan, Tom and Brown, Tom B and Chess, Benjamin and Child, Rewon and Gray, Scott and Radford, Alec and Wu, Jeffrey and Amodei, Dario},
  journal={arXiv preprint arXiv:2001.08361},
  year={2020}
}

@inproceedings{xiong2024effective,
  title={Effective long-context scaling of foundation models},
  author={Xiong, Wenhan and Liu, Jingyu and Molybog, Igor and Zhang, Hejia and Bhargava, Prajjwal and Hou, Rui and Martin, Louis and Rungta, Rashi and Sankararaman, Karthik Abinav and Oguz, Barlas and others},
  booktitle={Proceedings of the 2024 Conference of the North American Chapter of the Association for Computational Linguistics: Human Language Technologies (Volume 1: Long Papers)},
  pages={4643--4663},
  year={2024}
}

@article{hoffmann2022training,
  title={Training compute-optimal large language models},
  author={Hoffmann, Jordan and Borgeaud, Sebastian and Mensch, Arthur and Buchatskaya, Elena and Cai, Trevor and Rutherford, Eliza and Casas, DDL and Hendricks, Lisa Anne and Welbl, Johannes and Clark, Aidan and others},
  journal={arXiv preprint arXiv:2203.15556},
  volume={10},
  year={2022}
}

@article{guo2025deepseek,
  title={DeepSeek-R1 incentivizes reasoning in LLMs through reinforcement learning},
  author={Guo, Daya and Yang, Dejian and Zhang, Haowei and Song, Junxiao and Wang, Peiyi and Zhu, Qihao and Xu, Runxin and Zhang, Ruoyu and Ma, Shirong and Bi, Xiao and others},
  journal={Nature},
  volume={645},
  number={8081},
  pages={633--638},
  year={2025},
  publisher={Nature Publishing Group UK London}
}

@article{jaech2024openai,
  title={Openai o1 system card},
  author={Jaech, Aaron and Kalai, Adam and Lerer, Adam and Richardson, Adam and El-Kishky, Ahmed and Low, Aiden and Helyar, Alec and Madry, Aleksander and Beutel, Alex and Carney, Alex and others},
  journal={arXiv preprint arXiv:2412.16720},
  year={2024}
}

@article{evans2000broadcasting,
  title={Broadcasting on trees and the Ising model},
  author={Evans, William and Kenyon, Claire and Peres, Yuval and Schulman, Leonard J},
  journal={Annals of Applied Probability},
  pages={410--433},
  year={2000},
  publisher={JSTOR}
}

@inproceedings{sharan2018prediction,
  title={Prediction with a short memory},
  author={Sharan, Vatsal and Kakade, Sham and Liang, Percy and Valiant, Gregory},
  booktitle={Proceedings of the 50th Annual ACM SIGACT Symposium on Theory of Computing},
  pages={1074--1087},
  year={2018}
}

@article{ebeling1995long,
  title={Long-range correlations between letters and sentences in texts},
  author={Ebeling, Werner and Neiman, Alexander},
  journal={Physica A: Statistical Mechanics and its Applications},
  volume={215},
  number={3},
  pages={233--241},
  year={1995},
  publisher={Elsevier}
}

@article{shannon1948mathematical,
  title={A mathematical theory of communication},
  author={Shannon, Claude Elwood},
  journal={The Bell system technical journal},
  volume={27},
  number={3},
  pages={379--423},
  year={1948},
  publisher={Nokia Bell Labs}
}

@misc{jurafsky_martin,
  title={Speech and Language Processing: An Introduction to Natural Language Processing, Computational Linguistics, and Speech Recognition},
  author={Jurafsky, Daniel and Martin, James H},
  year={2026}
}

@article{elman1990finding,
  title={Finding structure in time},
  author={Elman, Jeffrey L},
  journal={Cognitive science},
  volume={14},
  number={2},
  pages={179--211},
  year={1990},
  publisher={Wiley Online Library}
}

@article{mossel-peres,
 ISSN = {10505164},
 URL = {http://www.jstor.org/stable/1193228},
 author = {Elchanan Mossel and Yuval Peres},
 journal = {The Annals of Applied Probability},
 number = {3},
 pages = {817--844},
 publisher = {Institute of Mathematical Statistics},
 title = {Information Flow on Trees},
 urldate = {2026-05-03},
 volume = {13},
 year = {2003}
}

@article{sly-colorings,
	author = {Sly, Allan},
	date = {2009/06/01},
	doi = {10.1007/s00220-009-0783-7},
	id = {Sly2009},
	isbn = {1432-0916},
	journal = {Communications in Mathematical Physics},
	number = {3},
	pages = {943--961},
	title = {Reconstruction of Random Colourings},
	url = {https://doi.org/10.1007/s00220-009-0783-7},
	volume = {288},
	year = {2009}}

@article{semerjian,
	author = {Semerjian, Guilhem},
	date = {2008/01/01},
	doi = {10.1007/s10955-007-9417-7},
	id = {Semerjian2008},
	isbn = {1572-9613},
	journal = {Journal of Statistical Physics},
	number = {2},
	pages = {251--293},
	title = {On the Freezing of Variables in Random Constraint Satisfaction Problems},
	url = {https://doi.org/10.1007/s10955-007-9417-7},
	volume = {130},
	year = {2008}}

@article{kesten-stigum,
author = {H. Kesten and B. P. Stigum},
title = {{Additional Limit Theorems for Indecomposable Multidimensional Galton-Watson Processes}},
volume = {37},
journal = {The Annals of Mathematical Statistics},
number = {6},
publisher = {Institute of Mathematical Statistics},
pages = {1463 -- 1481},
year = {1966},
doi = {10.1214/aoms/1177699139},
URL = {https://doi.org/10.1214/aoms/1177699139}
}

@article{bleher,
	author = {Bleher, P. M. and Ruiz, J. and Zagrebnov, V. A.},
	date = {1995/04/01},
	doi = {10.1007/BF02179399},
	id = {Bleher1995},
	isbn = {1572-9613},
	journal = {Journal of Statistical Physics},
	number = {1},
	pages = {473--482},
	title = {On the purity of the limiting gibbs state for the Ising model on the Bethe lattice},
	url = {https://doi.org/10.1007/BF02179399},
	volume = {79},
	year = {1995}}

@book{bloomfield1933language,
  author    = {Bloomfield, Leonard},
  title     = {Language},
  publisher = {Henry Holt and Company},
  address   = {New York},
  year      = {1933}
}

@article{wells1947immediate,
  author  = {Wells, Rulon S.},
  title   = {Immediate Constituents},
  journal = {Language},
  volume  = {23},
  number  = {2},
  pages   = {81--117},
  year    = {1947}
}

@book{tesniere1959elements,
  author    = {Tesni{\`e}re, Lucien},
  title     = {{\'E}l{\'e}ments de syntaxe structurale},
  publisher = {Klincksieck},
  address   = {Paris},
  year      = {1959}
}

@book{reed1877higher,
  author    = {Reed, Alonzo and Kellogg, Brainerd},
  title     = {Higher Lessons in English: A Work on English Grammar and Composition},
  publisher = {Clark and Maynard},
  address   = {New York},
  year      = {1877}
}

@book{chomsky1957syntactic,
  author    = {Chomsky, Noam},
  title     = {Syntactic Structures},
  publisher = {Mouton},
  address   = {The Hague},
  year      = {1957}
}

@article{montague1970universal,
  author  = {Montague, Richard},
  title   = {Universal Grammar},
  journal = {Theoria},
  volume  = {36},
  number  = {3},
  pages   = {373--398},
  year    = {1970}
}

@article{marcus1993building,
  author  = {Marcus, Mitchell P. and Santorini, Beatrice and Marcinkiewicz, Mary Ann},
  title   = {Building a Large Annotated Corpus of {English}: The {Penn Treebank}},
  journal = {Computational Linguistics},
  volume  = {19},
  number  = {2},
  pages   = {313--330},
  year    = {1993}
}

@inproceedings{socher2013recursive,
  author    = {Socher, Richard and Perelygin, Alex and Wu, Jean and Chuang, Jason and Manning, Christopher D. and Ng, Andrew Y. and Potts, Christopher},
  title     = {Recursive Deep Models for Semantic Compositionality Over a Sentiment Treebank},
  booktitle = {Proceedings of the 2013 Conference on Empirical Methods in Natural Language Processing (EMNLP)},
  pages     = {1631--1642},
  year      = {2013}
}

@inproceedings{tai2015improved,
  author    = {Tai, Kai Sheng and Socher, Richard and Manning, Christopher D.},
  title     = {Improved Semantic Representations from Tree-Structured Long Short-Term Memory Networks},
  booktitle = {Proceedings of the 53rd Annual Meeting of the Association for Computational Linguistics (ACL)},
  pages     = {1556--1566},
  year      = {2015}
}

@book{cover-thomas,
  author       = {Thomas M. Cover and
                  Joy A. Thomas},
  title        = {Elements of information theory {(2.} ed.)},
  publisher    = {Wiley},
  year         = {2006},
  url          = {http://www.elementsofinformationtheory.com/},
  isbn         = {978-0-471-24195-9},
  bibsource    = {dblp computer science bibliography, https://dblp.org}
}

@book{zwillinger1999crc,
  title={CRC standard probability and statistics tables and formulae},
  author={Zwillinger, Daniel and Kokoska, Stephen},
  year={1999},
  publisher={Crc Press}
}

@misc{modded_nanogpt_2024,
  author       = {Keller Jordan and Jeremy Bernstein and Brendan Rappazzo and
                  @fernbear.bsky.social and Boza Vlado and You Jiacheng and
                  Franz Cesista and Braden Koszarsky and @Grad62304977},
  title        = {modded-nanogpt: Speedrunning the NanoGPT baseline},
  year         = {2024},
  url          = {https://github.com/KellerJordan/modded-nanogpt}
}

@misc{jordan2024muon,
  author       = {Keller Jordan and Yuchen Jin and Vlado Boza and Jiacheng You and
                  Franz Cesista and Laker Newhouse and Jeremy Bernstein},
  title        = {Muon: An optimizer for hidden layers in neural networks},
  year         = {2024},
  url          = {https://kellerjordan.github.io/posts/muon/}
}

\appendix

\section{Overview of the Appendix}
For our theoretical results, we introduce the relevant background from Markov chains and the theory of broadcast processes in \Cref{sec:prelims}. In \Cref{sec:moments}, we carry out the asymptotic calculations for the ground truth Ising language using recursive arguments from the hierarchical structure. We build on this in \Cref{sec:moment-autoregressive} to make our theoretical predictions for autoregressive generation in the Ising (soft-constrained) version of our language. In \Cref{sec:freezing} we prove that for the coloring (hard-constrained) version of our language, short-context models without reasoning will typically produce invalid trees, building on the \emph{freezing} phenomenon in this model. In \Cref{sec:reasoning-proof}, we complete the proof of the positive theoretical result for the reasoning model. 

In \Cref{apdx:experiment} we give more details about the setup for our simulations with transformer models. In \Cref{apdx:auxiliary}, we state and prove some lemmas for asymptotic series calculations which are used in the proofs of our main results. 

\section{Further Preliminaries}
\label{sec:prelims}

\subsection{Markov Chains}
In many of our arguments, we will use the information-theoretic characterization of \emph{Markov chains} (see e.g.~\cite{cover-thomas}). Formally, a sequence of random variables $A_1,A_2,\ldots,A_n$ forms a Markov chain, denoted
\begin{equation*}
    A_1\to A_2\to \ldots\to A_n,
\end{equation*}
if for each $k$, conditioned on $A_k$, the random variables $(A_1,\ldots,A_{k-1})$ and $(A_{k+1},\ldots,A_n)$ are conditionally independent. In other words, the future and past of a Markov chain are conditionally independent given the present. An well-known equivalent definition is that the conditional distribution of $A_k$ given $A_1,\ldots,A_{k-1}$ is the same as the conditional distribution of $A_k$ given just $A_{k-1}$.

It will be convenient to track several induced Markov chains in the autoregressive Ising process as in \Cref{def:ar-broadcast} later on. We have the following highly convenient closed form for their cross-correlations for any Markov chain on marginally uniform spins:
\begin{lem}
\label{lem:binary-mc-correlations}
    Let $A_1\to \ldots\to A_n$ denote any Markov chain with each $A_i\in \{\pm 1\}$ marginally uniform. Let $\alpha_i=\mathbb{E}[A_iA_{i+1}]$. Then
    \begin{equation*}
        \mathbb{E}[A_1A_n]=\prod_{i=1}^{n-1} \alpha_i
    \end{equation*}
\end{lem}
\begin{proof}
    This immediately follows by induction using the Markov property: the claim is trivially true for $n=2$ and for general $n$, 
    \begin{align*}
        \mathbb{E}[A_1A_n]&=\mathbb{E}[\mathbb{E}[A_1A_n\vert A_{n-1}]]\\
        &=\mathbb{E}[\mathbb{E}[A_1\vert A_{n-1}]\mathbb{E}[A_n\vert A_{n-1}]]\\
        &=\alpha_{n-1}\mathbb{E}[A_{n-1}\mathbb{E}[A_1\vert A_{n-1}]]\\
        &=\alpha_{n-1}\mathbb{E}[\mathbb{E}[A_1A_{n-1}\vert A_{n-1}]]\\
        &=\alpha_{n-1}\mathbb{E}[A_1A_{{n-1}}],
    \end{align*}
    so we can conclude by induction. The second line follows by conditional independence in any Markov chain, while the rest is basic properties of binary random variables and conditional expectation.
\end{proof}

\subsection{Broadcast Processes}

In this section, we record simple calculations for $d$-ary trees as well as relevant background on broadcasting processes that will be important in our arguments.

\subsubsection{$d$-ary Trees}

For our computations, we will require the following distribution of least common ancestor heights in $T_{d,h}$.

\begin{prop}
\label{prop:tree-height-distribution}
    In $T_{d,h}$, for any leaf node $\ell$, the number of leaves in $T_{d,h}$ whose least common ancestor with $\ell$ is at distance $k$ is $1$ for $k=0$ and $(d^k-d^{k-1})$ for $1\leq k\leq h$.
\end{prop}
\begin{proof}
    This is immediate from the fact that all leaves at distance $k>0$ from $\ell$ are precisely the leaves belonging to the same subtree as $\ell$ of size $d^{k}$ but that do not belong to the same subtree of $r$ of size $d^{k-1}$.
\end{proof}

We will also need the following simple computation on the height of least common ancestors of a \emph{random adjacent pair of leaves} $(\ell,\ell+1)$ where $\ell$ is uniform on $[d]^h\setminus \{(d,\ldots,d)\}$. Recall that under our encoding of the leaves, this excluded vertex is simply the rightmost vertex of $T_{d,h}$.

\begin{prop}
\label{prop:neighbor-heights-geometric}
    Let $\mathcal{D}_h$ denote the (random) height of the least common ancestor of a random pair of adjacent leaves $(\ell,\ell+1)$ sampled uniformly in $T_{d,h}$. Then as $h\to \infty$, the law of $\mathcal{D}_h$ converges in distribution to $\mathsf{Geom}(1/d)$, a geometric random variable with success probability $1/d$.

    In particular, for any fixed $\rho\in [0,1)$, 
    \begin{equation*}
        \alpha_h:=\mathbb{E}_{s\sim \mathcal{D}_h}[\rho^{2s}]\to \alpha^*:=\mathbb{E}_{s\sim \mathsf{Geom}(1/d)}[\rho^{2s}]<1.
    \end{equation*}
\end{prop}
\begin{proof}
    Under our encoding of the tree, the least common ancestor of a pair $(\ell,\ell+1)\in ([d]^h)^2$ is precisely the last component of $\ell\in [d]^h\setminus \{(d,\ldots,d)\}$ that is not equal to $d$ since then there is no carryover. We can couple the sampling of $\ell$ uniform on $[d]^h\setminus \{(d,\ldots,d)\}$ with a uniform sample on $[d]^h$ up to error $1/d^h$ from sampling $\ell=(d,\ldots,d)$. As $h\to \infty$, this probability of failure tends to $0$, and moreover, the index of the last component that is not equal to $d$ converges to $\mathsf{Geom}(1/d)$ since each component then is equal to $d$ with probability $1/d$ independently across coordinates.
\end{proof}

\subsubsection{Ising Broadcast}
Consider the Ising broadcast process on $T_{d,h}$ with correlation parameter $\rho$ as in \Cref{def:broadcast}. We will view $d,\rho$ as fixed while $h\to \infty$. For our later computations to compute moments of the autoregressive Ising broadcast, we will need to define:
\begin{equation*}
    q_{h}:=q_{h}(d,\rho) = \mathbb{E}[\mathbb{E}[X\vert Y]^2]
\end{equation*}
where $(X,Y)$ is the root and leaves, respectively, of a sample from the broadcast process on $T_{d,h}$ with correlation $\rho$. In other words, this is the \emph{(squared) expected reconstruction advantage on average over the sampling of the leaves.} A fundamental fact from the theory of Ising broadcast processes is that when $d\rho^2<1$, then $q_h\to 0$ as $h\to \infty$ while above the Kesten-Stigum bound, the reconstruction probability stays bounded away from 0:

\begin{thm}[Kesten-Stigum Bound (\cite{kesten-stigum}), see e.g.~\cite{bleher}]
\label{thm:ks-bound}
Consider the Ising broadcast process on $T_{d,h}$ with correlation parameter $\rho$ and uniform prior on the root. Then if $d\rho^2>1$ is fixed, it holds that
\begin{equation*}
    \lim_{h\to \infty} q_h =q^*>0
\end{equation*}
for an absolute constant $q^*=q^*(d,\rho)$.
\end{thm}
\begin{proof}
    First, note that the limit exists because $q_h$ is \emph{monotone non-increasing} in $h$. This is because the broadcast process can be written as a Markov chain
    \begin{equation*}
        X\to Y^1\to Y^2\to \ldots\to Y^h,
    \end{equation*}
    where $Y^j$ is the vector of spins at level $j$ of the $d$-ary tree. In particular, the process for $t\geq 1$ defined by
    \begin{equation*}
        \mathbb{E}[X\vert Y^t]
    \end{equation*}
    forms a backwards martingale. In particular, the $L^2$-norms are decreasing since
    \begin{equation*}
        q_{t+1}=\mathbb{E}[\mathbb{E}[X\vert Y^{t+1}]^2]=\mathbb{E}[\mathbb{E}[\mathbb{E}[X\vert Y^t]\vert Y^{t+1}]^2]\leq \mathbb{E}[\mathbb{E}[\mathbb{E}[X\vert Y^t]^2\vert Y_{t+1}]]=\mathbb{E}[\mathbb{E}[X\vert Y^t]^2]=q_t,
    \end{equation*}
    by Cauchy-Schwarz and the tower law. Therefore, $q_h\to q^*$ for some $q^*\geq 0$ by nonnegativity.

    To see that $q^*$ is in fact strictly positive when above the Kesten-Stigum bound, it is well-known (see e.g. \cite{bleher}) that when $d\rho^2>1$, the following \emph{reconstruction} problem is solvable (\cite{mossel-peres}): 
    \begin{equation}
    \label{eq:ks-lower-bounded}
        \lim_{h\to \infty} d_{TV}(Y^h_1,Y^h_{-1})>0,
    \end{equation}
    where $Y^h_1,Y^h_{-1}$ denote the distributions of the leaves of the Ising broadcast process of depth $h$ \emph{conditioned} on the root $X$ being $+1$ and $-1$, respectively. By definition of total variation distance,
    \begin{align*}
        d_{TV}(Y^h_1,Y^h_{-1})&=\frac{1}{2}\sum_{Y^h\in \{\pm 1\}^{d^h}} \vert \Pr(Y^h\vert X=1)-\Pr(Y^h\vert X=-1)\vert\\
        &=\sum_{Y^h\in \{\pm 1\}^{d^h}} \Pr(Y^h)\vert \Pr( X=1\vert Y^h)-\Pr(X=-1\vert Y^h)\vert\\
        &=\mathbb{E}_{Y^h}\left[\vert \mathbb{E}[X\vert Y^h\right]\vert]\\
        &\leq \sqrt{q_h}.
    \end{align*}
    The second equality is Bayes' rule to rewrite the total variation distance in terms of the expected magnetization of the root given the leaves, and the final step is Cauchy-Schwarz. Since the total variation distance stays bounded away from zero above the KS bound by \Cref{eq:ks-lower-bounded}, the same therefore holds true for $q_h$ and hence the limiting value $q^*$.
\end{proof}

\subsubsection{Coloring Broadcast}

A well-known fact for the coloring broadcast process is the onset of \emph{freezing}: if the branching factor $d$ of $T_{d,h}$ is sufficiently large compared to the number of colors $q$, then with high probability over the sampling of the process, the posterior distribution of the root given the leaves is \emph{fixed} to a unique color. We will require the following quantitative version by \cite{sly-colorings} that was implicit in earlier work of \cite{mossel-peres} as well as \cite{semerjian}:

\begin{thm}[Freezing, Lemma 7 of~\cite{sly-colorings}]
\label{thm:frozen}
    Consider the coloring broadcast process on $T_{d,h}$ with $q$ colors. Then if the branching factor $d$ satisfies
    \begin{equation*}
        d\geq q(\log(q)+\log\log(q)+o_q(1)),
    \end{equation*}
    then if $X_L\in [q]^{d^h}$ denotes the colors of the leaves, then with probability at least $1-1/\log(q)$ over $X_L$:

    \begin{equation*}
        \mathrm{Law}(X_{\emptyset}\vert X_L) = \delta_i,
    \end{equation*}
    where $i\sim [q]$ is a marginally uniform color.
\end{thm}

\section{Moments of the Ising Language}\label{sec:moments}
In this section, we compute important statistics of the Ising broadcast process so as to compare with the statistics of the autoregressive version that uses bounded context windows in generation. In particular, we will compute various moments of these processes that will be used for the proof of \Cref{thm:log-linear-scaling} and \Cref{thm:kurtosis-main}.

\subsection{Moments of the Broadcast Process}

Throughout this section, we will consider the Ising broadcast process with correlation parameter $\rho$ satisfying $d\rho^2=\lambda>1$; we will view $d$ and $\rho$ as fixed. We remark that a well-known equivalent description of the broadcast process is that with probability $\rho$, an internal vertex copies its sign in $\Sigma$ to a child, while with probability $1-\rho$, it broadcasts a uniformly random sign.

To set up various recurrences, we will compute the following moments:
\begin{equation*}
    M_{d,\rho,h}(k):=\mathbb{E}\left[\left(\sum_{i\in [d]^{h}} X_i\right)^k\bigg\vert X_{\emptyset}=+1\right],
\end{equation*}
where $X_i$ for $i\in [d]^{h}$ are the nodes at level $h$ of the broadcast process and $X_{\emptyset}$ is the value of the root. While our focus will eventually $k=2,4$, this more general definition will enable us to recurse using the tree structure, and so here we will treat $h$ as a general variable in the recursion. Note that $M_{d,\rho,h}(k)$ for $k$ even is the same as the unconditional expectation by symmetry. We will find that these moments are equal to $(d\rho)^{kh}$ up to an explicit multiplicative constant depending only on $d,\rho$, but not $h$.

The first two moments are well-known and elementary:
\begin{prop}
\label{prop:first-second-moment}
    Let $X_L$ be sampled from the $(d,h)$-Ising broadcast process with fixed correlation $\rho$ satisfying $d\rho^2=\lambda>1$. Then 
    \begin{gather*}
        M_{d,\rho,h}(1) = (d\rho)^{h}\\
        M_{d,\rho,h}(2) = d^h+ d^h\cdot (1-1/d)\cdot \sum_{\ell=1}^h (d\rho^2)^{\ell}.
    \end{gather*}

    In particular, 
    \begin{equation*}
        M_{d,\rho,h}(2) = (1-o(1))\cdot C_{d,\rho}(2)\cdot (d\rho)^{2h},
    \end{equation*}
    where 
    \begin{equation*}
        C_{d,\rho}(2) := (1-1/d)\cdot \frac{d\rho^2}{d\rho^2-1}.
    \end{equation*}
\end{prop}
\begin{proof}
    The first moment is trivial from linearity of expectation. Each $X_r$ for $r\in [d]^h$ has the same marginal distribution by symmetry, and has conditional expectation $+1$ when the path of length $h$ from the root has no re-randomizations, while is otherwise mean zero. This means
    \begin{equation*}
        \mathbb{E}[X_r\vert X_{\emptyset}=+1] = \rho^h,
    \end{equation*}
    and so
    \begin{equation*}
        M_{d,\rho,h}(1) = d^h\cdot \rho^h = (d\rho)^h.
    \end{equation*}

    For the second moment, we can argue as follows:
    \begin{align*}
        M_{d,\rho,h}(2) &= \mathbb{E}\left[\left(\sum_{\ell\in [d]^{h}} X_{\ell}\right)^2\right]
        =\sum_{\ell,\ell'=1}^{d^h} \mathbb{E}[X_{\ell}X_{\ell'
        }]
        =\sum_{\ell,\ell'=1}^{d^h} \rho^{2 \mathsf{LCA}(\ell,\ell')}
        =d^h\cdot \sum_{\ell=1}^{d^h} \rho^{2 \mathsf{LCA}(1,\ell)}
    \end{align*}
    where $\mathsf{LCA}(\cdot,\cdot)$ denotes the height of the least common ancestor of the leaves. Here, we use the fact that $X_{\ell}$ and $X_{\ell'}$ have nontrivial correlation if and only if the signal from their least common ancestor propagates to both of them, which occurs with probability $\rho^{2\mathsf{LCA}(\ell,\ell')}$ by independence in the broadcasting along edges. The last equality is by symmetry of the tree. To evaluate this sum,  \Cref{prop:tree-height-distribution} shows that
    \begin{align*}
        \sum_{\ell=1}^{d^h} \rho^{2 \mathsf{LCA}(1,\ell)} &=1+\sum_{k=1}^h \rho^{2k}\cdot (d^{k}-d^{k-1})
        =1+(1-1/d)\cdot \sum_{k=1}^h \rho^{2k}\cdot d^{k},
    \end{align*}
    and the claim follows for $M_{d,\rho,h}(2)$. The final statement follows by noting that
    \begin{align*}
        d^h\cdot (1-1/d)\cdot \sum_{\ell=1}^h (d\rho^2)^{\ell}&= d^h\cdot (1-1/d)\cdot (d\rho^2)\cdot \left(\frac{(d\rho^2)^h-1}{d\rho^2-1}\right)\\
        &=(1-o(1))\cdot (1-1/d)\cdot\left(\frac{d\rho^2}{d\rho^2-1}\right)\cdot (d\rho)^{2h}.
    \end{align*}
    Since $d\rho^2>1$ is fixed, this term asymptotically dominates $d^h$ as $h\to \infty$.
\end{proof}

We now compute the higher moments we need using a recursive argument based on the lower depths and moments.

\begin{prop}
\label{prop:broadcast-third-moment}
    Let $X_L$ be sampled from the $(d,h)$-broadcast process with fixed correlation $\rho$ satisfying $d\rho^2=\lambda>1$. Then 
    \begin{gather*}
        M_{d,\rho,h}(3) = (1-o(1))\cdot C_{d,\rho}(3)\cdot (d\rho)^{3h},
    \end{gather*}
    where
    \begin{equation*}
        C_{d,\rho}(3) := \frac{1}{(d\rho)^2-1}\cdot \left(3\cdot \frac{(d-1)^2}{d}\cdot \frac{d\rho^2}{(d\rho^2-1)}+(d\rho)^2\cdot(1-1/d)(1-2/d)\right)
    \end{equation*}
\end{prop}
\begin{proof}
    Let us write $Z_1,\ldots,Z_d$ as the $d$ subtree sums of the children of the root node in the broadcast process. We then have:
\begin{align*}
    M_{d,\rho,h}(3) &= \mathbb{E}\left[\left(\sum_{i\in [d]} Z_i\right)^3\bigg\vert X_r=+1\right]\\
    &= d\rho M_{d,\rho,h-1}(3) + 3d(d-1)\rho M_{d,\rho,h-1}(2)\cdot M_{d,\rho,h-1}(1)+ d(d-1)(d-2) \rho^3 (M_{d,\rho,h-1}(1))^3\\
    &=d\rho M_{d,\rho,h-1}(3) + 3(d-1)M_{d,\rho,h-1}(2)\cdot (d\rho)^{h}+ (1-1/d)(1-2/d) (d\rho)^{3h}
\end{align*}
To see where these terms come from, note that:
\begin{enumerate}
    \item The first term corresponds to choosing the same subtree each time. In this case, with probability $\rho$, the root of the subtree remains the same as the root of the full tree and we get the same quantity of the subtree one level below. If this does not occur, then the root of the subtree is rerandomized and the expectation is zero by symmetry since we are computing an odd moment.

    \item The second term corresponds to the different ways of choosing two subtrees, one of them twice. There are $d\cdot (d-1)$ choices for them (where the first choice is taken twice), and given this choice, there are three permutations. The square term is precisely $M_{d,\rho,h-1}(2)$, while for the linear term, it again becomes mean zero with probability $1-\rho$ independently, else corresponds to the subtree sum conditioned on the subtree root remaining $1$. This has expectation $(d\rho)^{h-1}$.

    \item Finally, all three indices are different, and we may freely permute them to give $d(d-1)(d-2)$ options. Conditioned on the root, each retains the signal with probability $\rho$ independently, and given they all do, the corresponding means are $M_{d,\rho,h-1}(1)=(d\rho)^{h-1}$. When any of these events fails, the expectation simply becomes zero.
\end{enumerate}
The second equality then comes from plugging in \Cref{prop:first-second-moment} for the first moments.

Iterating this recurrence, we use \Cref{prop:first-second-moment} to obtain:
\begin{align*}
    M_{d,\rho,h}(3) &= (d\rho)^h+ 3(d-1)(d\rho)^h \sum_{\ell=0}^{h-1} M_{d,\rho,\ell}(2) +(1-1/d)(1-2/d)\sum_{\ell=0}^{h - 1} (d\rho)^{3h-2\ell}\\
        &= (d\rho)^h+ 3(d-1)(d\rho)^h (C_{d,\rho}(2)-o(1))\sum_{\ell=0}^{h-1} (d\rho)^{2\ell}  +(1-1/d)(1-2/d)(d\rho)^h\sum_{\ell=1}^h (d\rho)^{2\ell}\\
    &= (d\rho)^h + (1-o(1))\cdot 3\cdot (d-1)\cdot C_{d,\rho}(2)\cdot (d\rho)^h \frac{(d\rho)^{2h}-1}{(d\rho)^2 - 1}\\
    &\qquad +(1-1/d)(1-2/d)(d\rho)^h\frac{(d\rho)^{2h+2}-(d\rho)^2}{(d\rho)^2 -1}\\
    &= \left(1-o(1)\right)\cdot \frac{1}{(d\rho)^2-1}\cdot \left(3\cdot (d-1)\cdot C_{d,\rho}(2)+(d\rho)^2\cdot(1-1/d)(1-2/d)\right)\cdot (d\rho)^{3h}\\
    &:= (1-o(1))\cdot C_{d,\rho}(3)\cdot (d\rho)^{3h}
\end{align*}
In the second line, we use the elementary fact that if $x>1$ is fixed, then the nearly geometric sum $\sum_{i=0}^h (1-o(1))x^i=(1-o(1))\frac{x^{h+1}-1}{x-1}$ since this sum is dominated by the largest elements which have prefactors tending to $1$. The stated formula for $C_{d,\rho}(3)$ follows by substituting the value of $C_{d,\rho}(2)$.
\end{proof}

Finally, we can compute the fourth moment using the previous results:

\begin{prop}
\label{prop:broadcast-fourth-moment}
    Let $X_L$ be sampled from the $(d,h)$-broadcast process with fixed correlation $\rho$ satisfying $d\rho^2=\lambda>1$. Then 
    \begin{gather*}
        M_{d,\rho,h}(4) = (1-o(1))\cdot C_{d,\rho}(4)\cdot (d\rho)^{4h},
    \end{gather*}
    where
    \begin{equation*}
        C_{d,\rho}(4) := \frac{\left(4(d-1)\rho^2 C_{d,\rho}(3) + 3(d-1) (C_{d,\rho}(2))^2+ 6 (d-1)(d-2) \rho^2 C_{d,\rho}(2)+(d-1)(d-2)(d-3)\rho^4\right)}{{(d^3\rho^4)-1}}
    \end{equation*}
\end{prop}
\begin{proof}
    As in the previous computation, we again have
\begin{align*}
    M_{d,\rho,h}(4) &= \mathbb{E}\left[\left(\sum_{i\in [d]} Z_i\right)^4\bigg\vert X_r=+1\right]\\
    &= d M_{d,\rho,h-1}(4) + 4d(d-1)\rho^2 M_{d,\rho,h-1}(3)\cdot M_{d,\rho,h-1}(1)
    +3 d (d-1) (M_{d,\rho,h-1}(2))^2\\
    &\qquad +6 d(d-1)(d-2) \rho^2 M_{d,\rho,h-1}(2) \cdot (M_{d,\rho,h-1}(1))^2\\
    &\qquad + d(d-1)(d-2)(d-3) \rho^4 (M_{d,\rho,h-1}(1))^4
\end{align*}
The interpretation of each of the terms is similar, grouping by the possible exponents of the at most four distinct indices. One can verify that the coefficients indeed sum to $d^4$. We pick up a $\rho$ factor for each term that corresponds to an odd moment of a subtree since the root copies the signal with this probability, or else the odd moment is zero by symmetry.

Unwinding this recurrence by expanding $M_{d,\rho,h - 1}(4)$, then expanding $M_{d,\rho,h - 2}(4)$, and so on, we find that %$M_{d,\rho,h}(4)$ equals
\begin{align*}
    M_{d,\rho,h}(4) = d^h &+ 4(d-1)\rho^2 \sum_{\ell=0}^{h-1} d^{h-\ell} M_{d,\rho,\ell}(3) M_{d,\rho,\ell}(1)+3 (d-1)\sum_{\ell=0}^{h-1} d^{h-\ell}(M_{d,\rho,\ell}(2))^2\\
    &+ 6 (d-1)(d-2)\rho^2\sum_{\ell=0}^{h-1} d^{h-\ell}M_{d,\rho,\ell}(2)\cdot (M_{d,\rho,\ell}(1))^2 \\
    &+  (d-1)(d-2)(d-3) \rho^4 \sum_{\ell=0}^{h-1} d^{h-\ell}(M_{d,\rho,\ell}(1))^4
\end{align*}
and simplifying yields that $M_{d,\rho,h}(4)$ equals
\begin{align*}
    &d^h + (1-o(1))\cdot 4\cdot d^h(d-1)\rho^2 C_{d,\rho}(3)\sum_{\ell=0}^{h-1}  d^{-\ell}(d\rho)^{4\ell} +(1-o(1))3d^h (d-1)(C_{d,\rho}(2))^2\sum_{\ell=0}^{h-1} d^{-\ell}(d\rho)^{4\ell}\\
    &\quad + (1-o(1))6 d^h(d-1)(d-2) \rho^2 C_{d,\rho}(2) \sum_{\ell=0}^{h-1} d^{-\ell}(d\rho)^{4\ell} + d^h (d-1)(d-2)(d-3) \rho^4 \sum_{\ell=0}^{h-1} d^{-\ell}(d\rho)^{4\ell}\\
    &= (1-o(1))\cdot  d^h\\
    &\quad \cdot\left(4(d-1)\rho^2 C_{d,\rho}(3) + 3(d-1) (C_{d,\rho}(2))^2+ 6 (d-1)(d-2)\rho^2 C_{d,\rho}(2)+(d-1)(d-2)(d-3)\rho^4\right)\\
    &\quad \cdot \frac{(d^3\rho^4)^{h}-1}{(d^3\rho^4)-1}\\
    &= (1-o(1))\cdot  (d\rho)^{4h}\cdot\\
    &\underbrace{\frac{\left(4(d-1)\rho^2 C_{d,\rho}(3) + 3(d-1) (C_{d,\rho}(2))^2+ 6 (d-1)(d-2) \rho^2 C_{d,\rho}(2)+(d-1)(d-2)(d-3)\rho^4\right)}{{(d^3\rho^4)-1}}}_{:=C_{d,\rho}(4)}
\end{align*}
as claimed.
\end{proof}

\section{Moments of the Autoregressive Ising Process}\label{sec:moment-autoregressive}

In this section, we complete the proof of \Cref{thm:log-linear-scaling} and \Cref{thm:kurtosis-main}. To do this, we use the results of the previous section to compute moments for the \emph{autoregressive} Ising process as defined in \Cref{def:ar-broadcast}. The main difference is that the cross-correlations will decay much faster even when above the Kesten-Stigum bound, so that the moments add up almost independently.

Recall that to generate the leaves of a new subtree $Y_{i}$ of size $d^w$ in the autoregressive process given $Y_{i-1}$, we sample the height $h_i$ from the law of heights of adjacent subtrees at depth $h-w$ in $T_{d,h}$ and then sample $Y_{i}$ from the conditional law on $(Y_{i-1},Y_{i})$ from the broadcast process when these subtrees have least common ancestor at height $h_i$. This distribution over heights was considered in \Cref{prop:neighbor-heights-geometric}.

Define $\alpha = \alpha_{h-w} = \mathbb{E}[\rho^{2h'}]$, where $h'$ is sampled from the random height distribution $\mathcal{D}_{h-w}$ of adjacent subtrees at depth $h-w$. By \Cref{prop:neighbor-heights-geometric}, we know that $\alpha\to \alpha^*$ as defined there as $h-w\to \infty$. Throughout the following results, we similarly define 
\begin{equation*}
    Z_i:=\sum_{\ell=1}^{d^w} (Y_i)_{\ell}
\end{equation*}
for the sum of the leaves in the $i$th generated subtree of size $d^w$.

\subsection{Equivalent Markov Chains}

For our computations, we will use the following equivalent description of the AR process from the $i$th subtree leaves $Y_i$ to the $i+j$th subtree leaves $Y_{i+j}$. Because $Y_i$ is marginally a sample from the normal broadcast process, we can equivalently sample this process via the following Markov chain:
\begin{equation}
\label{eq:ar-sequence-mc}
    Y_i\to X_i\to X_{i+1}\to Y_{i+1}\to X_{i+1}'\to X_{i+2}\to Y_{i+2}\to \ldots\to Y_{i+j-1}\to X'_{i+j-1}\to X_{i+j}\to Y_{i+j},
\end{equation}
where $X_{\ell}$ is the root of the $\ell$th subtree; note that the actual AR process is only the restriction to the $Y_{\ell}$ part. To sample $X_i$ as well as each $X_{i+\ell}'$ for $\ell\geq 1$, we sample the posterior of the root of the corresponding subtree given the leaves in the previous step of the Markov chain (i.e. $Y_{i}$ and $Y_{i+\ell}$, respectively). The transition $X_i\to X_{i+1}$ as well as each $X'_{i+\ell}\to X_{i+\ell+1}$ for $j\geq 1$ is obtained by sampling the height $h_{i+\ell}$ of the LCA of the subtrees and then rerandomizing the spin except with probability $\rho^{2h_{i+\ell}}$. Taken on average over $h_{i+\ell}\sim \mathcal{D}_{h-w}$, this means that the probability of not rerandomizing is precisely $\mathbb{E}[X'_{i+\ell}X_{i+\ell+1}]=\alpha$, which we know converges to $\alpha^*$ as $h-w\to \infty$ by \Cref{prop:neighbor-heights-geometric}.\footnote{Note that the difference in the first step of this Markov chain comes from the fact that conditioned on $h_i$, the marginal on $(Y_i,Y_{i+1})$ is the same as an honest sample from the broadcast process from a root that connects to the root of each subtree at distance $h_i$ in each branch; in other words, we need not sample $X_i'$ since conditioned on $Y_i$, it has the same law as $X_i$ itself so we may rewrite in this way.}

\iffalse For the block autoregressive process as above for a tree of overall depth $h$, we can evaluate $\alpha$ fairly explicitly. The random height of the least common ancestor of two \emph{adjacent} leaves in $T_{d,w}$ can be obtained by sampling the first leaf index uniformly from the set $[d]^w\setminus \{(d,\ldots,d)\}$. Under the natural $d$-ary encoding of the leaves, the height of the least common ancestor with the next leaf is obtained by finding the last element of this leaf encoding that is not equal to $d$. When $w\to \infty$, the law of the random heights $h_i$ converge to a geometric random variable with success probability $1/d$.
\fi

 \subsection{Preliminary Calculations}
In the remaining computations, we will leverage the following highly convenient fact: for a pair of random variables $(X_1,X_2)\in \{\pm 1\}^2$ that are marginally uniform, if $\mathbb{E}[X_1X_2]=\gamma\geq 0$, then one can equivalently sample them by sampling $X_1$ uniformly and then setting $X_2=X_1$ with probability $\gamma$ and otherwise setting $X_2$ to be a uniform bit. We will refer to this latter event, which occurs with $1-\gamma$ probability, as \emph{rerandomizing}.

Next, we will calculate the correlations of the root values of the generated subtrees in the autoregressive process as expressed in the equivalent Markov chain given in \Cref{eq:ar-sequence-mc}.
\begin{lem}
\label{lem:ar-root-markov-chain-corr}
    Fix any $i\geq 1$ and let $j\geq 1$. Then for the Markov chain given in \Cref{eq:ar-sequence-mc},
    \begin{equation*}
        \mathbb{E}[X_iX_{i+j}] = \alpha(\alpha q_w)^{j-1}.
    \end{equation*}
\end{lem}

\begin{proof}
     If $j=1$, then by the Markov representation in \Cref{eq:ar-sequence-mc}, we have:
    \begin{equation*}
        Y_i\to X_i\to X_{i+1}\to Y_{i+1}
    \end{equation*}
    is a Markov chain, and in particular, the probability that there is no rerandomization from $X_i\to X_{i+1}$ is $\alpha$. Since these bits are marginally uniform, this means $\mathbb{E}[X_iX_{i+1}]=\alpha$.

    For $j>1$, observe that we can consider the restriction in \Cref{eq:ar-sequence-mc} to the Markov chain on subtree root spins in $\{\pm 1\}$ given by 
    \begin{equation*}
        X_i\to X_{i+1}\to X_{i+2}\to \ldots\to X_{i+j}.
    \end{equation*}
    Since these spins are binary and marginally uniform, \Cref{lem:binary-mc-correlations} implies that
    \begin{equation*}
        \mathbb{E}[X_iX_{i+j}]=\prod_{k=1}^{j} \alpha_k,
    \end{equation*}
    where $\alpha_k = \mathbb{E}[X_{i+k-1}X_{i+k}]$. Therefore, it suffices for us to compute $\alpha_k$ for $k\geq 2$ since we already know $\alpha_1=\alpha$.

    To do this, observe that we again have the restriction to a Markov chain on $\{\pm 1\}$ with uniform marginals:
    \begin{equation*}
        X_{i+k-1}\to X'_{i+k-1}\to X_{i+k},
    \end{equation*}
    where by the above discussion,
    \begin{equation*}
        \mathbb{E}[X'_{i+k-1}X_{i+k}]=\mathbb{E}[\rho^{2h_{i+k-1}}]=\alpha
    \end{equation*}
    since this comes only from the sampling of the random height $h_{i+k}$ and then rerandomizing except with probability $\rho^{2h_{i+k}}$ between the roots of these adjacent subtrees.
    By \Cref{lem:binary-mc-correlations}, it therefore suffices to compute
    \begin{equation*}
        \mathbb{E}[X_{i+k-1}X_{i+k-1}']
    \end{equation*}
    from the AR process, which we can write as the induced Markov chain $X_{i+k-1}\to Y_{i+k-1}\to X'_{i+k-1}$. To do this, observe that conditional on $Y_{i+k-1}$, $X_{i+k-1}$ and $X_{i+k-1}'$ are independent posterior samples of the subtree root value $X$ given these leaves, and so
    \begin{align*}
        \mathbb{E}[X_{i+k-1}X_{i+k-1}']&=\mathbb{E}[\mathbb{E}[X_{i+k-1}X_{i+k-1}'\vert Y_{i+k-1}]]\\
        &=\mathbb{E}[\mathbb{E}[X_{i+k-1}\vert Y_{i+k-1}]^2]\\
        &=q_w.
    \end{align*}
    Putting these two steps together, \Cref{lem:binary-mc-correlations} therefore implies that for $k>1$,
    \begin{align*}
        \alpha_k &= \mathbb{E}[X_{i+k-1}X_{i+k}]\\
        &=\alpha\cdot q_{w},
    \end{align*}
    which completes the proof.
\end{proof}

We now use this explicit expression for the correlations of root spins in the AR process to calculate the correlations of subtree sums by leveraging conditional independencies once we condition on these spins:

\begin{lem}
\label{lem:ar-correlation-1-1}
    In the autoregressive broadcast process, for any $i,j\geq 1$, it holds that
    \begin{equation*}
        \mathbb{E}[Y_iY_{i+j}] = M_{d,\rho,w}(1)^2\cdot \alpha\cdot (\alpha\cdot q_w)^{j-1}.
    \end{equation*}
\end{lem}
\begin{proof}
    We have that 
    \begin{equation*}
        Z_i\to X_i\to X_{i+j}\to Z_{i+j}
    \end{equation*}
    is a Markov chain. Since $\mathbb{E}[X_iX_{i+j}]=\alpha\cdot (\alpha\cdot q_w)^{j-1}$ by \Cref{lem:ar-root-markov-chain-corr}, the probability that the transition $X_i\to X_{i+j}$ does not rerandomize is precisely $\alpha\cdot (\alpha\cdot q_w)^{j-1}$; moreover, note that this event is independent of the actual value of $X_i$ by symmetry and therefore also $Z_i$ since the process is a Markov chain. If $\mathcal{E}$ is the event that this rerandomization event occurs, then by the tower law
    \begin{equation*}
        \mathbb{E}[Z_iZ_{i+j}\vert \mathcal{E}]= \mathbb{E}[Z_i\mathbb{E}[Z_{i+j}\vert \mathcal{E},Y_i]\vert \mathcal{E}] = 0
    \end{equation*}
    since the subtree root $X_{i+j}$ is then a uniform root bit conditioned on $Y_i$ and $\mathcal{E}$. On the other hand, when $\mathcal{E}$ does not occur, we know that $X_i=X_{i+j}$. Then
    \begin{equation*}
        \mathbb{E}[Z_iZ_j\vert \mathcal{E}^c]=\mathbb{E}[Z_iZ_j\vert X_i=X_j]= X_i^2(M_{d,\rho,w}(1))^2=(M_{d,\rho,w}(1))^2,
    \end{equation*}
    since in this case, the Markov property implies $Y_i,Y_{i+j}$ are conditionally independent samples from the broadcast process but with common root value $X_i$. The claim then follows by the law of total probability as
    \begin{equation*}
        \mathbb{E}[Z_iZ_{i+j}] = \Pr(\mathcal{E})\cdot 0+ \Pr(\mathcal{E}^c)\cdot (M_{d,\rho}(1))^2 = M_{d,\rho}(1)^2\cdot \alpha\cdot (\alpha\cdot q_w)^{j-1}.
    \end{equation*}
\end{proof}

We now compute the higher cross-moments that will be needed for calculating the fourth moment of the autoregressive process. The first bound is trivial:

\begin{lem}
\label{lem:ar-correlation-diagonal}
    In the AR process, for any $i$, it holds that
    \begin{gather*}
        \mathbb{E}[Z_i^2] = M_{d,w,\rho}(2)\\
        \mathbb{E}[Z_i^4] = M_{d,w,\rho}(4)
    \end{gather*}
\end{lem}
\begin{proof}
This is trivial since $Y_i$ has the same distribution as the usual broadcast process of depth $w$.
\end{proof}

\begin{lem}
\label{lem:ar-correlations-2-1-1}
    In the AR process, for any $i<j<k$, it holds that
    \begin{equation*}
        \mathbb{E}[Z_i^2Z_jZ_k] = \alpha\cdot (\alpha q_w)^{k-j-1}M_{d,\rho,w}(2)\cdot M_{d,\rho,w}(1)^2.
    \end{equation*}
    Similarly, it holds that
    \begin{equation*}
        \mathbb{E}[Z_iZ_jZ_k^2] = \alpha\cdot (\alpha q_w)^{k-j-1}M_{d,\rho,w}(2)\cdot M_{d,\rho,w}(1)^2
    \end{equation*}
\end{lem}
\begin{proof}
    We claim that conditioned on $Z_i^2$, $(Y_j,Y_k)$ has the same law as in the AR process without any further conditioning. Indeed, note that conditioning on $Y_i^2$ does not change the distribution of the root $X_i$ by symmetry since the square makes the sum independent of the root value. This implies that the rest of the Markov chain has identical conditional distribution to the original Markov chain. Therefore, 
    \begin{equation*}
        \mathbb{E}[Z_i^2Z_jZ_k]=\mathbb{E}[Z_i^2]\mathbb{E}[Z_jZ_k],
    \end{equation*}
    and so, we may conclude by \Cref{lem:ar-correlation-1-1} and the fact that $Y_i$ itself is simply a sample from the broadcast process with depth $w$. An analogous argument holds for the second expectation since similarly, the distribution of $Z_k^2$ is independent of $(Y_i,Y_j)$ by symmetry after squaring.
\end{proof}

\begin{lem}
\label{lem:ar-correlation-1-1-1-1}
    In the AR process, for any $i<j<k<\ell$, it holds that
    \begin{equation*}
        \mathbb{E}[Z_iZ_jZ_kZ_{\ell}] = \alpha^2\cdot (\alpha q_w)^{(\ell-k)+(j-i)-2}\cdot M_{d,\rho,w}(1)^4.
    \end{equation*}
\end{lem}
\begin{proof}
    Note that the restriction of the AR process
    \begin{equation*}
        Y_i\to  Y_j\to X_j'\to Y_k\to Y_{\ell}.
    \end{equation*}
    remains a Markov chain. To compute the expectation, we first condition on $X_j'$ to obtain
    \begin{equation*}
        \mathbb{E}[\mathbb{E}[Z_iZ_j\vert X_j']\mathbb{E}[ Z_kZ_{\ell}\vert X_j']],
    \end{equation*}
    where we use the Markov property to decompose the conditionally independent parts as the subtree sums are functions of the corresponding leaves. However, observe that by symmetry, both expectations are independent of the value of $X_j'$ since as in the proof of \Cref{lem:ar-correlation-1-1}, the expectation is zero except on the event that there is no rerandomization. This is independent of the spin $X_j'$ by symmetry. In particular, we find that
    \begin{equation*}
        \mathbb{E}[Z_iZ_jZ_kZ_{\ell}] = \mathbb{E}[Z_iZ_j]\mathbb{E}[ Z_kZ_{\ell}],
    \end{equation*}
    and we conclude by \Cref{lem:ar-correlation-1-1}.
\end{proof}

\begin{lem}
\label{lem:ar-correlation-1-2-1}
    In the AR process, for any $i<j<k$, it holds that
    \begin{equation*}
        0\leq \mathbb{E}[Z_iZ_j^2Z_k] \leq \alpha^2\cdot (\alpha q_w)^{(k-j)+(j-i)-2}\cdot M_{d,\rho,w}(1)^2\cdot M_{d,\rho,w}(2).
    \end{equation*}
\end{lem}
\begin{proof}
We again have a Markov chain
\begin{equation*}
    Y_i\to X_i\to X_{j}\to  Y_j\to X'_j\to X_k\to  Y_k.
\end{equation*}
We claim that conditional on $Y_j$, the probability there is no rerandomization between $X_i$ and $X_k$ is 
\begin{equation*}
    \alpha^2\cdot (\alpha q_w)^{(k-j)+(j-i)-2}\cdot \mathbb{E}[X_j\vert Y_j]^2.
\end{equation*}
Indeed, the conditional probability that there is no randomization between $X_i\to X_j$ and from $X'_j\to X_k$ is precisely $\alpha^2\cdot (\alpha q_w)^{(k-j)+(j-i)-2}$ by \Cref{lem:ar-root-markov-chain-corr} since these events are independent of the actual value of $Y_j$ by the Markov property and spin symmetry. Moreover, the probability that there is no randomization between $X_j$ and $X'_j$ given $Y_j$ is $\mathbb{E}[X_j\vert Y_j]^2$ by the argument of \Cref{lem:ar-root-markov-chain-corr}. If we let $\mathcal{E}$ denote the event that there is rerandomization between $X_i$ and $X_k$, then similar computations to before imply
\begin{align*}
    \mathbb{E}[Z_iZ_j^2Z_k]&=\mathbb{E}[\mathbb{E}[Z_iZ_j^2Z_k\vert Y_j]]\\
    &=\mathbb{E}[Z_j^2\mathbb{E}[Z_iZ_k\vert Y_j]]\\
    &=M_{d,\rho,w}(1)^2\cdot \mathbb{E}[Z_j^2\Pr(\mathcal{E}^c\vert Z_j)]\\
    &=\alpha^2\cdot (\alpha q_w)^{(k-j)+(j-i)-2}\cdot M_{d,\rho,w}(1)^2\cdot \mathbb{E}[Z_j^2\mathbb{E}[X_j\vert Y_j]^2]\\
    &\leq \alpha^2\cdot (\alpha q_w)^{(k-j)+(j-i)-2}\cdot M_{d,\rho,w}(1)^2\cdot M_{d,\rho,w}(2).
\end{align*}
In the last step, we use the trivial fact that $X_j\in \{-1,1\}$ to pull out the factor. Note that this computation implies that the expectation is nonnegative since it admits a square representation.
\end{proof}

\begin{lem}
    In the AR process, for any $i<j$, it holds that
    \begin{equation*}
        \mathbb{E}[Z_i^2Z_j^2] = M_{d,w,\rho}(2)^2.
    \end{equation*}
\end{lem}
\begin{proof}
We again have a Markov chain
\begin{equation*}
    Y_i\to X_i\to X_j \to Y_j.
\end{equation*}
Conditioned on $(X_i,X_j)$, $Y_i$ and $Y_j$ are conditionally independent, and moreover, the law of $(Z_i^2,Z_j^2)$ is independent of $(X_i,X_j)$ by symmetry since the subtree sums get squared and thus do not depend on the root values. The claim follows.
\end{proof}

\begin{lem}
\label{lem:ar-correlation-3-1}
    In the AR process, for any $i<j$, it holds that
    \begin{equation*}
        \mathbb{E}[Z_i^3Z_j] = \mathbb{E}[Z_iZ_j^3]= \alpha\cdot (\alpha\cdot q_w)^{j-i-1}\cdot M_{d,\rho,w}(3)\cdot M_{d,\rho,w}(1).
    \end{equation*}
\end{lem}
\begin{proof}
We again have a Markov chain
\begin{equation*}
    Y_i\to X_i\to X_j \to Y_j.
\end{equation*}
By similar reasoning to \Cref{lem:ar-correlation-1-1}, let $\mathcal{E}$ denote the event that there is a rerandomization from $X_i\to X_j$. By \Cref{lem:ar-root-markov-chain-corr}, the complementary event $\mathcal{E}^c$ has probability $\alpha\cdot (\alpha\cdot q_w)^{j-i-1}$ even conditioned on $X_i$. We again know that the conditional expectation of $Z_j$ given $\mathcal{E}$ occurs and given $Y_i$ is zero, and moreover,
\begin{equation*}
    \mathbb{E}[Z_i^3Z_j\vert \mathcal{E}^c,X_i] = X_i^2\cdot M_{d,\rho,w}(3)\cdot M_{d,\rho,w}(1)=M_{d,\rho,w}(3)\cdot M_{d,\rho,w}(1),
\end{equation*}
since on this event, the roots are identical and then one can apply conditional independence conditioned on having the same root. The claim again follows.
\end{proof}

\subsection{Variance and Fourth Moments}

At this point, we can use the cross-correlations computed in the previous section to compute the variance of the sum in the autoregressive process:

\begin{prop}
\label{prop:ar-variance}
    Let $X_L$ be sampled from the $(d,h)$-Ising autoregressive broadcast process with fixed correlation $\rho$ satisfying $d\rho^2>1$ with $w \to \infty$ and $h - w \to \infty$. Then 
    \begin{align*}
        \mathsf{Var}\left(\sum_{i=1}^{d^h} (X_L)_i\right)=\mathsf{Var}\left(\sum_{i=1}^{d^{h-w}} Z_i\right)
    &=d^{h-w}M_{d,w}(2)+(1-o(1))\cdot \frac{2\alpha^*}{1-\alpha^*q^*}\cdot (d\rho)^{2w}\cdot d^{h-w}\\
    &= (1-o(1))\cdot A_{d,\rho}(2)\cdot (d\rho)^{2w}\cdot d^{h-w}
    \end{align*}
    where 
    \begin{equation*}
        A_{d,\rho}(2):=\left(C_{d,\rho}(2)+\frac{2\alpha^*}{1-\alpha^*q^*}\right)
    \end{equation*}
\end{prop}
\begin{proof}
    We can directly compute:
    \begin{align*}
        \mathsf{Var}\left(\sum_{i=1}^{d^{h-w}} Z_i\right)=\sum_{i,j=1}^{d^{h-w}} \mathbb{E}[Z_iZ_j]
        &=\sum_{i=1}^{d^{h-w}} \mathbb{E}[Z_i^2]+2\sum_{1\leq i<j\leq d^{h-w}} \mathbb{E}[Z_iZ_j]\\
        &=d^{h-w}M_{d,w}(2)+2(d\rho)^{2w}\cdot \alpha \sum_{\ell=0}^{d^{h-w}-1} (d^{h-w}-\ell)(\alpha\cdot q_w)^{\ell}
    \end{align*}
    In the third equality, we sum over distances between pairs and use \Cref{lem:ar-correlation-1-1}. 
    
    To obtain the final asymptotics, note that since we are above the Kesten-Stigum bound ($d\rho^2>1$), \Cref{thm:ks-bound} implies that $q_w\to q^*\in (0,1]$ as $w\to \infty$ and we further have $\alpha\to \alpha^*\in (0,1)$ since $h-w\to \infty$ by \Cref{prop:neighbor-heights-geometric}. We can therefore apply dominated convergence and standard geometric identities/bounds to deduce
    \begin{equation*}
        2(d\rho)^{2w}\cdot \alpha \sum_{\ell=0}^{d^{h-w}-1} (d^{h-w}-\ell)(\alpha\cdot q_w)^{\ell}=(1-o(1))\cdot \frac{2\alpha^*}{1-\alpha^*q^*} \cdot (d\rho)^{2w}\cdot d^{h-w}.
    \end{equation*}
    \Cref{prop:first-second-moment} shows that as $w\to \infty$, since $d\rho^2>1$,
    \begin{equation*}
        M_{d,w}(2) = (1-o(1))\cdot C_{d,\rho}(2)\cdot (d\rho)^{2w},
    \end{equation*}
    so we obtain the stated claim.
\end{proof}

We thus obtain \Cref{thm:log-linear-scaling}:
\begin{thm}[\Cref{thm:log-linear-scaling}, restated]
Consider the Ising broadcast process on $T_{d,h}$ with $d\rho^2>1$. For a given context depth $0\leq w\leq h$, let $X_L\in \{\pm 1\}^{d^h}$ denote a sample from the autoregressive process on the leaves with context length $d^w$. Then if $w \to \infty$ and $h-w\to \infty$, the \emph{(log-) normalized variance} satisfies:
    \begin{equation}
        \log\left(\frac{1}{d^h}\mathsf{Var}\left(\sum_{i=1}^{d^h} (X_L)_i\right)\right) =  w\log(d\rho^2)+\log(A_{d,\rho}(2))+o(1).
    \end{equation}

By contrast, when $X_L\in \{\pm 1\}^{d^h}$ is a sample from the \emph{true} Ising broadcast process (i.e. $w=h$), 
\begin{equation*}
    \log\left(\frac{1}{d^h}\mathsf{Var}\left(\sum_{i=1}^{d^h} (X_L)_i\right)\right) =  h\log(d\rho^2)+\log(C_{d,\rho}(2))+o(1).
\end{equation*}
\end{thm} 
\begin{proof}
The first identity for the autoregressive process is obtained by taking the logarithm in \Cref{prop:ar-variance}. The second identity for the true language was already obtained in \Cref{prop:first-second-moment}.
\end{proof}

\iffalse
using \Cref{prop:tree-height-distribution}: the sampled height is $\ell=1,\ldots, h-w$ with probability $d^{\ell-1}(d-1)/(d^{h-w}-1)$, so

\begin{align*}
    \alpha &= \frac{(d-1)}{d^{h-w}-1}\sum_{\ell=1}^{h-w}\rho^{2\ell}d^{\ell-1}\\
    &=\frac{(d-1)\rho^2}{d^{h-w}-1}\sum_{\ell=0}^{h-w-1}\rho^{2\ell}d^{\ell}\\
    &=\frac{(d-1)\rho^2}{d^{h-w}-1}\cdot \frac{(d\rho^2)^{h-w}-1}{d\rho^2-1}\\
    &=\left(1+\frac{1-\rho^2}{d\rho^2-1}\right)\left(\frac{(d\rho^2)^{h-w}-1}{d^{h-w}-1}\right).
\end{align*}
Note that in the polynomial scaling where $h-w = (1-\theta) h$, the second factor scales like a constant times $\rho^{2(1-\theta) h}$ so in particular is miniscule. In this case, the second term in the variance is significantly smaller than the dominant term of the variance, which is $\Theta(d^{h-w}(d\rho)^{2w})$.
\fi

We can now compute the fourth moment, which will imply \Cref{thm:kurtosis-main}:

\begin{prop}
\label{prop:ar-fourth-moment}
    Let $X_L$ be sampled from the $(d,h)$-Ising autoregressive broadcast process with fixed correlation $\rho$ satisfying $d\rho^2>1$ and context depth $w\leq h$. Then if $h-w, w\to\infty$,
    \begin{align*}
\mathbb{E}\left[\left(\sum_{i=1}^{d^{h-w}} Z_i\right)^4\right]
    &=(1-o(1)) \cdot A_{d,\rho}(4)\cdot d^{2(h-w)} \cdot (d\rho)^{4w},
    \end{align*}
\end{prop}
where 
\begin{equation*}
    A_{d,\rho}(4) = 3\cdot\left(C_{d,\rho}(2)+\frac{2\alpha^*}{1-\alpha^*q^*}\right)^2=3\cdot A_{d,\rho}(2)^2.
\end{equation*}
\begin{proof}
We expand by the indices in order by counting the types of terms:
\begin{align*}
    \mathbb{E}\left[\left(\sum_{i=1}^{d^{h-w}} Z_i\right)^4\right] 
    &= d^{h-w}\cdot \mathbb{E}[Z_1^4]+ \sum_{1\leq i<j\leq d^{h-w}} \left(6\mathbb{E}[Z_i^2Z_j^2]+4\mathbb{E}[Z_i^3 Z_j]+4\mathbb{E}[Z_j^3 Z_i]\right)\\
    &\quad +12\sum_{1\leq i<j<k\leq d^{h-w}} \left( \mathbb{E}[Z_i^2Z_jZ_k]+ \mathbb{E}[Z_iZ_j^2Z_k]+ \mathbb{E}[Z_iZ_jZ_k^2]\right)\\
    &\quad +24\sum_{1\leq i<j<k<\ell\leq d^{h-w}}\mathbb{E}[Z_iZ_jZ_kZ_{\ell}]\\
    &= d^{h-w}\cdot M_{d,w}(4) +\sum_{1\leq i<j\leq d^{h-w}} \left(6M_{d,w}(2)^2+8M_{d,w}(3)\cdot M_{d,w}(1)\cdot \alpha(\alpha\cdot q_w)^{j-i-1}\right)\\
    &\quad +12(M_{d,w}(2))(M_{d,w}(1))^2\cdot \alpha\sum_{1\leq i<j<k\leq d^{h-w}} \left((\alpha\cdot q_w)^{k-j-1}+(\alpha\cdot q_w)^{j-i-1}\right)\\
    &\quad +12\sum_{1\leq i<j<k\leq d^{h-w}} \mathbb{E}[Z_iZ_j^2Z_k]\\
    &\quad +24(M_{d,w}(1))^4\alpha^2\sum_{1\leq i<j<k<\ell\leq d^{h - w}} (\alpha\cdot q_w)^{(\ell-k)+(j-i)-2}.
\end{align*}

Here, we applied \Cref{lem:ar-correlation-diagonal}, \Cref{lem:ar-correlation-3-1}, \Cref{lem:ar-correlations-2-1-1}, and \Cref{lem:ar-correlation-1-1-1-1} to express all terms exactly in terms of $\alpha,q_w$ except for the $Z_iZ_j^2Z_k$ terms. We will momentarily show that this extra term is of asymptotically of lower order.

For any sequence $\beta_w\to \beta<1$ as $s\to \infty$, we have the following identities:

\begin{gather*}
    \sum_{1\leq i<j\leq d^{h-w}} \beta_w^{j-i-1}=(1-o(1)) \cdot\frac{1}{1-\beta}\cdot d^{h-w}\\
    \sum_{1\leq i<j<k\leq d^{h-w}} \beta_w^{k-j-1}+\beta^{j-i-1}=(1-o(1))\cdot \frac{1}{1-\beta}\cdot d^{2(h-w)}\\
    \sum_{1\leq i<j<k<\ell\leq d^{h-w}} \beta_w^{(\ell - k-1) + (j-i-1)}=(1-o(1)) \cdot\frac{1}{2}\cdot \left(\frac{1}{1-\beta}\right)^2\cdot d^{2(h-w)}.
\end{gather*}
These first approximate identity is easy to see directly by counting the number of possibilities for $j-i=s$ for each $s$ and using dominated convergence, while the other two are from \Cref{lem:fourth-moment-k-j} and \Cref{lem:fourth-moment-l-k-j-i}.

Recalling that for $a=2,3,4$ and $w\to \infty$, we have $M_{d,\rho,w}(a)=(1-o(1))\cdot C_{d,\rho}(a)\cdot (d\rho)^{aw}$ by \Cref{prop:first-second-moment}, \Cref{prop:broadcast-third-moment}, and \Cref{prop:broadcast-fourth-moment}, the dominant order terms are those of order $d^{2(h-w)}\cdot (d\rho)^{4w}$. Plugging in the above identities with $\beta_w = \alpha_wq_w\to \alpha^*q^*<1$ and keeping just these highest-order terms, we deduce that
\begin{align*}
    \mathbb{E}\left[\left(\sum_{i=1}^{d^{h-w}} Y_i\right)^4\right] &= (1-o(1))\cdot \left(3 C_{d,\rho}(2)^2+ 12 C_{d,\rho}(2)\cdot \frac{\alpha^*}{1-\alpha^*q^*}+12\left(\frac{\alpha^*}{1-\alpha^*q^*}\right)^2\right)\cdot d^{2(h-w)}\cdot (d\rho)^{4w}\\
    &=(1-o(1))\cdot 3\cdot  A_{d,\rho}(2)^2\cdot d^{2(h-w)}\cdot (d\rho)^{4w}, 
\end{align*}
as claimed.

It remains to confirm that the final term indeed satisfies
\begin{equation*}
    \sum_{1\leq i<j<k\leq d^{h-w}} \mathbb{E}[Z_iZ_j^2Z_k]=O\left(d^{h-w}\cdot (d\rho)^{4w}\right)
\end{equation*}
so that it is indeed asymptotically negligible as $w,h-w\to \infty$. To do this, \Cref{lem:ar-correlation-1-2-1} provides the bounds
\begin{equation*}
    0\leq \mathbb{E}[Z_iZ_j^2Z_k]\leq \alpha^2\cdot (\alpha q_w)^{(k-j)+(j-i)-2}\cdot M_{d,\rho,w}(1)^2\cdot M_{d,\rho,w}(2).
\end{equation*}
Since $q_w\to q^*>0$ as we are above the Kesten-Stigum bound by \Cref{thm:ks-bound} and also $0<\alpha^*q^*<1$, \Cref{lem:fourth-moment-k-i} shows that the sum over all tuples $1\leq i<j<k\leq d^{h-w}$ is $O(d^{h-w})\cdot M_{d,\rho,w}(1)^2\cdot M_{d,\rho,w}(2)=O(d^{h-w}\cdot (d\rho)^{4w})$ as $w,h-w\to \infty$, confirming this term is indeed lower-order as claimed.
\end{proof}

We immediately deduce \Cref{thm:kurtosis-main}:

\begin{thm}[\Cref{thm:kurtosis-main}, restate]
    Consider the Ising broadcast process on $T_{d,h}$ with $d\rho^2>1$. For a given context depth $0\leq w\leq h$, let $X_L\in \{\pm 1\}^{d^h}$ denote a sample from the autoregressive process on the leaves with context length $d^w$. Then if $w \to \infty$ and $h-w\to \infty$, the excess kurtosis of $\sum_i (X_L)_i$ satisfies:
    \begin{equation*}
    \Kurt\left(\sum_{i=1}^{d^h} (X_L)_i\right)-3:=\frac{\mathbb{E}\left[\left(\sum_{i=1}^{d^h} (X_L)_i\right)^4\right]}{\mathbb{E}\left[\left(\sum_{i=1}^{d^h} (X_L)_i\right)^2\right]^2}-3=o(1)\,.
    \end{equation*}
\end{thm}
\begin{proof}
    Simply combine the second moment of the autoregressive process in \Cref{prop:ar-variance} and the fourth moment in \Cref{prop:ar-fourth-moment}.
\end{proof}

\section{Inconsistent Colorings Above Freezing}\label{sec:freezing}
We now turn to establishing that an autoregressive broadcast process for the coloring process will fail to produce a \emph{consistent} coloring of the full $d$-ary tree. By this, we mean that with high probability over the sampling of the autoregressive process, there is \emph{no} proper coloring of the entire tree $T_{d,h}$ that can be extended from the color of the sampled leaves. In fact, we will show that this  holds even just to go up \emph{a single level} of the tree.

To prove this, we require the following simple observation on the posterior distribution of a child of the root given any prior distribution on any other child:
\begin{lem}
\label{lem:ar-lower-bound}
    Consider the coloring broadcast process with $q$ colors on $T_{d,h}$. Let $X_{\emptyset}$ denote the color of the root of the tree and let $X_1,\ldots, X_d$ denote the colors of the children. Then given any distribution $\nu$ on the first child $X_1$, the posterior distribution on $X_2$ is lower bounded by $(q-2)/(q-1)^2$ for all colors.
\end{lem}
\begin{proof}
    It suffices to show that the claim is true conditioned on any value of $X_1$. Given the color of $X_1\in [q]$, the posterior distribution on $X_{\emptyset}$ is simply uniform on $[q]\setminus \{X_1\}$. Conditioned on any of these valid colors, the posterior distribution on $X_2$ simply becomes uniform on $[q]\setminus \{X_{\emptyset}\}$. It is easy then to see that for any $i\in [q]$, and initial choice of $X_1$, there are $q-2$ choices of $X_{\emptyset}$ such the posterior probability $X_i=i$ is $1/(q-1)$. Since each of these $q-2$ choices for $X_{\emptyset}$ has probability at least $1/(q-1)$ given $X_1$, the claim follows.
\end{proof}

We can now put this simple observation with the onset of freezing:
\begin{thm}[\Cref{thm:freezing-inconsistent}, restated]
    Consider the autoregressive coloring broadcast process with $q$ colors on $T_{d,h}$ with any context depth $w<h$ such that $d>q(\log(q)+\log\log(q)+1+o_q(1))$. Then with probability $1-o_q(1)$, the sampling of the leaves in the autoregressive coloring process is inconsistent with any proper coloring of $T_{d,h}$.
\end{thm}
\begin{proof}
    Consider an internal vertex $v$ at depth $h-w-1\geq 0$. We will show that in the autoregressive coloring process for this range of parameters, it holds with $1-o_q(1)$ probability over the sampling of the $d$ subtrees of size $d^w$ that there is no valid color for vertex $v$.

    Label the children of $v$ by $1,\ldots,d$ at level $h-w$. In the autoregressive sampling, we sample the subtree of $i+1$ of size $d^w$ given the colors of the subtree of $i$. 
    It is clear that marginally, the subtree of vertex $1$ is a sample from the true coloring broadcast process of the $d$-ary tree at depth $w$, which can be obtained by sampling the color $X_1$ uniformly and then running the broadcast process for this root value. We will consider the following equivalent description of the sampling process: to sample the $(i+1)$th subtree given the value of the leaves of the $i$th subtree, we obtain the posterior distribution on $X_i$, which implies the posterior on $X_{i+1}$, which is then broadcast to the leaves in the coloring process.
    
    By \Cref{thm:frozen}, it follows that in this sampling process that $X_1$ is frozen to a uniformly random color $i_1\in [q]$ given the colors of its leaves with probability at least $1-1/\log(q)$. More generally, using \Cref{thm:frozen} again and symmetry, the posterior of each $X_i$ given the leaves of the $i$th subtree is again frozen to the actual value $X_i$ given the colors of the leaves with probability at least $1-1/\log(q)$, and by \Cref{lem:ar-lower-bound}, it further holds that each color $X_i$ has probability at least $(q-2)/(q-1)^2$ of being any $i\in [q]$ for any conditioning of the previous subtree (which has an induced posterior $\nu$ on $X_{i-1}$). 

    We now claim that with $1-o_q(1)$ probability, this implies that there is no valid color for vertex $v$ that can be consistent with the colors of the $d$ sampled subtrees of this process. Indeed, this event is implied by the event that for each color $i\in [q]$, there exists some $X_j$ in this process that is frozen to color $i\in [q]$ by the corresponding leaves of the subtree. By a standard coupon collector argument,
    \begin{align*}
    \Pr&\left(\exists i\in [q]: \text{$i$ valid for $v$ given leaves}\right) \\
        &\leq q\left(1-(1-1/\log(q))\frac{q-2}{(q-1)^2}\right)^d\\
        &\leq q\exp\left(-(1-O(1/q^2))\cdot\frac{d(1-1/\log(q))}{q}\right)\\
        &\leq q\exp\left(-(1-O(1/q^2))\cdot(\log(q)+\log\log(q)+1+o_q(1))(1-1/\log(q))\right)\\
        &\leq q\exp(-(1-O(1/q^2))(\log(q)+\log\log(q)-o_q(1)))\\
        &=\frac{1+o_q(1)}{\log(q)},
    \end{align*}
    as claimed.
\end{proof}

\begin{rmk}
It appears unlikely that one can go below the freezing threshold and still obtain the same inconsistent colorings. The reason is simply because this bound is tight for freezing, and as soon as one is below it, the probability of freezing decays doubly exponentially in the height of the tree. This can be most easily seen when $d=q-1$, in which case the probability a parent at depth $h-1$ is frozen is on the order of $e^{-q}$. Following the recursion, this causes the probability of a node at depth $h-w$ being inconsistent with the leaves to be $\exp(-\Theta(q^w))$, so it is actually \emph{doubly exponentially small} in the depth. This can only be offset by the number of vertices for very small context depths $w$ to obtain an internal node that has no valid colors.
\end{rmk}

\section{Exponential Advantage of Reasoning}\label{sec:reasoning-proof}

In this section, we give a proof of \Cref{thm:reasoning}. We design an even simpler model where $p(\cdot\mid -)$ conditions only on the current memory state. Set $\mathcal{M}=(\Sigma^{h-w}\cup\{\emptyset\})\times[d]^{h-w}$. The $\Sigma^{h-w}\cup\{\emptyset\}$ component encodes the values along the path from the root of the entire tree to the root of the current subtree, and the $[d]^{h-w}$ component encodes the index of the next subtree. We start with an initial memory state $(\emptyset,(1,\cdots,1))$. Let $M=(P, r)$ be the current memory state, and we describe a procedure that samples the next leaves $Y'$ and the next memory state $M'=(P',r')$. The idea is to mimic the depth-first search of the tree defining the $(d,h,\kappa)$-language.
\begin{enumerate}
    \item We first sample, if not already observed, the values along the path defined by $r$ in the current $(d,h,\kappa)$-broadcast process.
    \begin{itemize}
        \item If $r=(1,\cdots,1)$, then we recursively sample through the broadcast channel $\kappa$
        \[
            X_{\emptyset},X_{(1)},X_{(1,1)},\cdots,X_{r}\,.
        \]
        That is, $X_{\emptyset}$ is sampled from the stationary distribution of $\kappa$, and $X_{r[1:\ell+1]}$ is sampled from $\kappa(X_{r[1:\ell]},\cdot)$ for each $\ell=0,1,\cdots,h-w-1$.
    
        \item Otherwise, let $j$ be the largest integer such that $r_j>1$. By the induction hypothesis which will be ensured later, the $j$th component of $P$ equals the already observed $X_{r[1:j-1]}$ ($X_{\emptyset}$ if $j=1$). Starting from this, we recursively sample through the broadcast channel $\kappa$
        \[
            X_{r[1:j]},X_{r[1:j+1]},\cdots,X_{r}\,.
        \]
    \end{itemize}

    \item Now we have access to the values $X_{\emptyset},\cdots,X_r$. We sample a $(d,w,\kappa,\delta_{X_r})$-language and set $Y'$ to be the sampled leaves. This is the broadcast process on $T_{d,w}$ where the root has a fixed value $X_r$. Then we set
    \[
        P'=(X_{\emptyset},\cdots,X_{r[1:h-w-1]})\,.
    \]

    \item Let $k$ be the largest integer such that $r_k<d$. We set
    \[
        r'=(r_1,\cdots,r_k+1,1,\cdots,1)\,.
    \]
    If such an integer does not exist, the sampling has finished and we reset to $r'=(1,\cdots,1)$.
\end{enumerate}
As noted before, this can be understood as the depth-first sampling of a broadcast process and thus samples the correct language by construction.

\section{Experiment Details}\label{apdx:experiment}

\subsection{Tokenization}\label{apdx:tokenization}

We elaborate on a tokenization procedure of our synthetic language based on the broadcast model on $d$-ary trees, as described in \Cref{def:broadcast}. The first step is to tokenize the states in $\Sigma$. Since the two broadcast processes we consider have discrete state spaces, we can simply use them as tokens. In addition, we insert ``punctuation'' tokens to encode the information of the least common ancestor of the two adjacent nodes. In other words, we use the token set $\mathcal{T}=\Sigma\sqcup\mathcal{P}$ where $\mathcal{P}=\{p_1,p_2,\cdots,p_{h-1}\}$ is the set of punctuation tokens. Before we mathematically define the tokenization procedure, we refer the reader to \Cref{fig:ternary} for a quick illustration on how we tokenize an Ising broadcast model.

The tokenization of $T_{d,h}$ with leaves $X_r$ for $r\in[d]^{h}$ is a sequence
\[
    \tau_1,\tau_2,\cdots,\tau_{L}
\]
of length $L=d^{h-1}(d+1)-1$. The $m$th token $\tau_m$ is determined by the following. We write
\[
    m=r_0+(d+1)\sum_{i=1}^{h-1}(r_i-1)d^{i-1}
\]
where $r_i\in[d]$ for $1\leq i\le h-1$ and $0\leq r_0\leq d$. Note that there is a unique such expansion. Now $\tau_m$ is given by:
\[
    \tau_m=\begin{cases}
        X_{(r_{h-1},\cdots,r_0)}&\text{if $r_0\geq1$,}\\
        p_{\zeta(m)}&\text{if $r_0=0$}
    \end{cases}
\]
where
\[
    \zeta(m)=\max\{i\in\mathbb{Z}^+:\text{$r_j=1$ for all $1\leq j<i$}\}\,.
\]

\subsection{Training Without Reasoning}\label{apdx:train-vanilla}

We describe a procedure of preparing a training set to train a simple autoregressive model. Namely, given the context length $k$, we explain how the pair of an input sequence $\mathbf{x}$ and an output sequence $\mathbf{y}$ is sampled. The model is trained using the usual cross entropy loss.

For efficient training, we add an additional ``context refresh'' token $\emptyset$ to the token set $\tilde{\mathcal{T}}=\mathcal{T}\cup\{\emptyset\}$. This can also be understood as the ``beginning of sequence'' token. Now we sample $\mathbf{x},\mathbf{y}\in\tilde{\mathcal{T}}^k$ with the following algorithm.
\begin{enumerate}
    \item Sample an infinite sequence of independent broadcast processes: $T_{d,h}^{(1)}$ with leaves $X_r^{(1)}$, $T_{d,h}^{(2)}$ with leaves $X_r^{(2)}$, and so on.
    \item Tokenize the processes which gives the sequences
    \[
        (\tau_1^{(1)},\cdots,\tau_L^{(1)}), (\tau_1^{(2)},\cdots,\tau_L^{(2)}),\cdots\,.
    \]
    \item Join them with the context refresh token, giving an infinite sequence
    \[
        (\tilde{\tau}_1,\tilde{\tau}_2,\cdots)=(\tau_1^{(1)},\cdots,\tau_L^{(1)},\emptyset,\tau_1^{(2)},\cdots,\tau_L^{(2)},\emptyset,\tau_1^{(3)},\cdots)\,.
    \]
    \item Sample a random index $\iota\sim\operatorname{Uniform}(\{1,\cdots,L+1\})$ and obtain a consecutive subsequence
    \[
        \tilde{\tau}_{\iota},\tilde{\tau}_{\iota+1},\cdots,\tilde{\tau}_{\iota+k}\,.
    \]
    In particular, $\tilde{\tau}_{\iota}=\tau_{\iota}^{(1)}$ if $\iota\leq L$ and $\tilde{\tau}_{\iota}=\emptyset$ otherwise.
    \item Now set
    \[
        \mathbf{x}=(\tilde{\tau}_{\iota},\cdots,\tilde{\tau}_{\iota+k-1})\,,\qquad\mathbf{y}=(\tilde{\tau}_{\iota+1},\cdots,\tilde{\tau}_{\iota+k})\,.
    \]
\end{enumerate}
For each backpropagation, a fresh pair of $(\mathbf{x},\mathbf{y})$ is sampled from this procedure. This resembles the way the pre-training of large language models is usually done on large corpuses. In practice, the sequence $(\tilde{\tau}_1,\cdots,)$ and the index $\iota$ are not randomly sampled for each training step, but often $\iota$ is just taken sequentially from a fixed sequence. In our experiments, however, the differences between the two approaches were insignificant.

The inference (sampling) step of the trained model is straightforward: we ``prompt'' the model with the context refresh token $\emptyset$, sample the next token conditioned on the previous $k$ tokens, repeating $L$ times.

\subsection{Training With Reasoning}\label{apdx:train-reasoning}

As briefly noted in \Cref{sec:simulations}, the reasoning models are trained by inserting the ``memory tokens'' inbetween value tokens in the training set. Before we elaborate on this, it is helpful to describe the inference step, i.e., design a model ``wrapper'' that hides the memory tokens and outputs the value tokens in $\mathcal{T}$.

We first choose positive integers $\ell_{v}$ and $\ell_m$ such that $2\ell_m+\ell_v-1$ is at most the context size $k$. We want the model to retain a memory of size $\ell_m$ and output a sequence of $\ell_v$ value tokens (including punctuation tokens) each step. To be specific, we want the model to perform the following sampling step: given its current memory state represented by $\ell_m$ memory tokens, output $\ell_v$ value tokens and $\ell_m$ memory tokens representing the next memory state. This is why we require $2\ell_m+\ell_v\leq k+1$. We repeat this until the model has outputted $L$ tokens.

\paragraph{Markov chain of memory states.} The proof of \Cref{thm:reasoning} in \Cref{sec:reasoning-proof} describes a way to design the memory state using the path from the root to the current position. The conditional distribution function constructed in the proof can be understood as the transition kernel of a Markov chain of memory states, where each transition from one state to another outputs a subtree of height $w$. Our first step is to build a similar Markov chain so that it outputs a token \emph{one at a time}. This requires a bit of extension since the model must be able to output a punctuation token. We set the memory states as the alternating sequences of values and indices of the form
\begin{equation}\label{eqn:reasoning-memory}
    \sigma_1,r_1,\sigma_2,r_2,\cdots,\sigma_{h-h_0},r_{h-h_0}
\end{equation}
where $\sigma_i\in\Sigma$, $r_i\in[d]$, and $0\leq h_0\leq h$. This in particular means that the model has sampled
\[
    (X_{\emptyset},X_{r_1},X_{(r_1,r_2)},\cdots,X_{(r_1,\cdots,r_{h-h_0-1})})=(\sigma_1,\cdots,\sigma_{h-h_0})\,.
\]
The state transition is defined as follows, depending on the value of $h_0$.
\begin{itemize}
    \item The case $h_0=0$ is further divided into the following two.
    \begin{itemize}
        \item If $r_h<d$, then the model samples and outputs $X_{(r_1,\cdots,r_{h}+1)}$ from $X_{(r_1,\cdots,r_{h-1})}=\sigma_h$ through the broadcast channel $\kappa$. The next state is a sequence
        \[
            \sigma_1,r_1,\sigma_2,r_2,\cdots,\sigma_{h},r_h+1
        \]
        of the same length $2h$.

        \item If $r_h=d$, then the model computes the largest index $j<h$ such that $r_j<d$. If no such $j$ exists (i.e., $r_i=d$ for all $i$), then the model outputs a context refresh token $\emptyset$ and jumps to the state $\emptyset$. Otherwise, the model outputs a punctuation token $p_{h-j}$ and jumps to the state
        \[
            \sigma_1,r_1,\cdots,\sigma_j,r_j
        \]
        which is a sequence of length $2j$.
    \end{itemize}

    \item If $0<h_0<h$, our construction guarantees $r_{h-h_0}<d$. The model recursively samples
    \[
        X_{(r_1,\cdots,r_{h-h_0}+1)},X_{(r_1,\cdots,r_{h-h_0}+1,1)},\cdots,X_{(r_1,\cdots,r_{h-h_0}+1,1,\cdots,1)}
    \]
    from $X_{(r_1,\cdots,r_{h-h_0-1})}=\sigma_{h-h_0}$ through $\kappa$, where the subscript of the last element has $h_0$ ones. The model outputs $X_{(r_1,\cdots,r_{h-h_0}+1,1,\cdots,1)}$ and jumps to the state
    \[
        \sigma_1,r_1,\cdots,\sigma_{h-h_0},r_{h-h_0}+1,X_{(r_1,\cdots,r_{h-h_0}+1)},1,\cdots,X_{(r_1,\cdots,r_{h-h_0}+1,1,\cdots,1)},1
    \]
    where the subscript of the penultimate element has $h_0-1$ ones.

    \item The case $h_0=h$ corresponds to the initial and final state $\emptyset$. The model samples
    \[
        X_{\emptyset},X_{1},X_{(1,1)},\cdots,X_{(1,\cdots,1)}
    \]
    from a fresh tree, outputs the last value whose subscript has $h$ ones, and jump to the state
    \[
        X_{\emptyset},1,\cdots,X_{(1,\cdots,1)},1
    \]
    where the subscript of the penultimate element has $h-1$ ones.
\end{itemize}
Similar to the the argument in \Cref{sec:reasoning-proof}, $d^{h-1}(d+1)-1$ state transitions starting from the initial memory state $\emptyset$ outputs the correct language distribution, including the punctuation marks.

\paragraph{Tokenization of the memory states.} Tokenizing the memory states \Cref{eqn:reasoning-memory} is straightforward. We use the token set
\[
    \mathcal{T}_m=\Sigma\sqcup[d]\sqcup\{\emptyset\}\sqcup\{s,e\}
\]
where $\emptyset$ is the optional ``padding'' token, which allows us to tokenize the memory states in a fixed length $2h$ as
\[
    \sigma_1,r_1,\cdots,\sigma_{h-h_0},r_{h-h_0},\emptyset,\emptyset,\cdots,\emptyset,\emptyset
\]
using $2(h-h_0)$ padding tokens. To prevent the model from confusing the memory state tokens with the output tokens, one might also choose to use different set of tokens for $\Sigma$ and $\emptyset$. As discussed in \Cref{sec:simulations}, we also enclose the memory tokens with the additional ``memory start token'' $s$ and ``memory end token'' $e$.

Recall that we want the model to generate $\ell_m$ memory tokens every $\ell_v$ output value tokens. Following the above tokenization of the memory states, we need $\ell_m=2+2h$ tokens. As a consequence, we require that $1\leq\ell_v\leq k-2h-1$.

\paragraph{Constructing training sequences.} If the values of all of the nodes (including the internal nodes) of a tree $T_{d,h}$ are given, then the Markov chain described in the previous paragraph becomes deterministic. In other words, it gives a unique sequence of memory states
\[
    \emptyset=M_0,M_1,\cdots,M_L,M_{L+1}=\emptyset\,.
\]
Now each training sequence is generated from the following procedure.
\begin{enumerate}
    \item Sample an infinite sequence of independent broadcast processes $T_{d,h}^{(1)},T_{d,h}^{(2)}$ and the corresponding sequences of memory states: for each $T_{d,h}^{(i)}$ we write
    \[
        \emptyset=M_0^{(i)},M_1^{(i)},\cdots,M_{L+1}^{(i)}=\emptyset
    \]
    for its sequence of memory states.
    
    \item Following the same process in \Cref{apdx:train-vanilla} we obtain the sequence
    \[
        (\tilde{\tau}_1,\tilde{\tau}_2,\cdots)=(\tau_1^{(1)},\cdots,\tau_L^{(1)},\emptyset,\tau_1^{(2)},\cdots,\tau_L^{(2)},\emptyset,\tau_1^{(3)},\cdots)\,.
    \]

    \item Sample a random index $\iota\sim\operatorname{Uniform}(\{1,\cdots,L+1\})$ and consider the infinite sequence
    \[
        \tilde{\tau}_{\iota},\tilde{\tau}_{\iota+1},\cdots\,.
    \]
    \item We insert $\ell_m=2+2h$ tokens right before each of the tokens $\tilde{\tau}_{\iota},\tilde{\tau}_{\iota+\ell_v},\tilde{\tau}_{\iota+2\ell_v},\cdots$. The memory tokens to insert before $\tilde{\tau}_{s}$ are determined by the following. If $s-1=(L+1)i+j$ for integers $i$ and $0\leq j\leq L$, we insert the tokenization of $M_j^{(i+1)}$.

    \item Take the first $k+1$ tokens from the new token sequence. Set $\mathbf{x}$ and $\mathbf{y}$ to be the first $k$ and the last $k$ tokens from the chosen subsequence.
\end{enumerate}
Note that $\mathbf{x}$ is an alternating sequence of $\ell_m$ memory tokens and $\ell_v$ output value tokens. Thus, the model learns to predict the next tokens given the first $\ell_m$ memory tokens, and to maintain its memory state every $\ell_v$ outputs.

\subsection{Computing Resources and Reproducibility}\label{apdx:computation}

\paragraph{Model hyperparameters.} Our base model nanochat, in its default configuration, automatically chooses its own model hyperparameters given the desired number of transformer layers and the context size. We use $10$ transformer layers for all of our experiments which is slightly larger than our maximum tree height $8$. The weights of a model occupy less than 200 megabytes of disk storage.

\paragraph{Training iterations.} We configure to train $2^{15}=32768$ tokens in one mini-batch and $2^{19}=524288$ tokens in one iteration. For instance, if the context size is $2^{10}=1024$, then $2^5=32$ training sequences $(\mathbf{x},\mathbf{y})$ are processed in one optimizer step and one training iteration always consists of $2^4=16$ optimizer steps. We run 3,000 to 10,000 training iterations depending on the model we train.

\paragraph{Optimizer.} The optimizer we use is a combination of a modified Adam by \cite{modded_nanogpt_2024} and Muon by \cite{jordan2024muon}. This is taken from the optimizer of the original nanochat by \cite{nanochat} without modification.

\paragraph{Computation.} We note that nanochat is powerful but lightweight enough to reach a GPT-2 grade capability for natural language in around 2 hours training on 8xH100 GPU node. Due to the simplicity of our synthetic language, all of the pre-training runs in our experiments take a few hours even with 2 A100 GPUs, which makes them easily reproducible. Approximately, 1,000 training iterations take about 1 hour of 2 A100 GPUs.

\paragraph{Context size for ``Simulated'' data points.}
We note that for the ``Simulated'' data points in \Cref{fig:ising-plots,fig:coloring-plot}, the context size is defined to be $d^{w-1}(d+1)$, instead of $d^w$, because we properly take the punctuation marks into account.

\paragraph{Estimating variance and kurtosis. } Given the i.i.d. samples, these are computed using the standard sample variance and sample excess kurtosis formulae commonly available in scientific computing packages (see, e.g., \cite{zwillinger1999crc}). %Note that normal distributions and the Rademacher distribution have excess kurtosis $0$ and $-2$, respectively.

% \TODO

\section{Auxiliary Results}\label{apdx:auxiliary}

\begin{lem}
\label{lem:fourth-moment-k-i}
    Let $\alpha_n\in [0,1)$ be a sequence that converges to some $\alpha^*<1$. Then 
    \begin{equation*}
        \sum_{1\leq i<j<k\leq n} \alpha_n^{k-i} = (1-o(1))\left(\frac{\alpha^*}{1-\alpha^*}\right)^2\cdot n.
    \end{equation*}
\end{lem}
\begin{proof}
    Let $s:=k-i$ and note that $s$ ranges from $2$ to $n-2$ under our index convention. For convenience, we will write $\alpha$ instead of $\alpha_n$ below. Then by an elementary counting argument,
    \begin{equation*}
        \sum_{1\leq i<j<k\leq n} \alpha^{k-i} = \sum_{s=2}^{n-2} (n-s)(s-1)\alpha^s,
    \end{equation*}
    since there are $(n-s)$ valid $i<k$ pairs with this difference, and for each such pair, there are exactly $s-1$ valid indices $j$ for any such pair. We deduce that
    \begin{align*}
         \sum_{1\leq i<j<k\leq n} \alpha^{k-i} = n\sum_{s=2}^{n-2} \alpha^s(s-1) - \sum_{s=2}^{n-2}
         s(s-1)\alpha^s.
         \end{align*}
         By the dominated convergence theorem (to be precise, $\alpha_n<\alpha'<1$ for all sufficiently large $n$ by the convergence assumption and the geometric series is bounded), the first sum converges with 
         \begin{equation*}
             \lim_{n\to \infty} \sum_{s=2}^{n-2} \alpha^s (s-1) = \sum_{s=2}^{\infty} (s-1)(\alpha^*)^s=\left(\frac{\alpha^*}{1-\alpha^*}\right)^2,
         \end{equation*}
         and similarly the second sum converges to a constant by standard bounds on geometric moments.
\end{proof}

\begin{lem}
\label{lem:fourth-moment-k-j}
    Let $\alpha_n\in [0,1)$ be a sequence that converges to some $\alpha^*<1$. Then 
    \begin{equation*}
        \sum_{1\leq i<j<k\leq n} \alpha_n^{k-j} = (1-o(1))\cdot \frac{\alpha^* n^2}{2(1-\alpha^*)}
    \end{equation*}.
\end{lem}
\begin{proof}
    We again will write $\alpha$ in place of $\alpha_n$. By a similar combinatorial argument, we may write
    \begin{equation*}
        \sum_{1\leq i<j<k\leq n} \alpha^{k-j} = \sum_{s=1}^{n-1} \left(\sum_{i=1}^{n-s-1} i\right)\alpha^s=\frac{1}{2}\sum_{s=1}^{n-1} (n-s)(n-s-1)\alpha^s,
    \end{equation*}
    since there are $(n-s)$ valid $j<k$ pairs with difference $s$, and one may take all possible indices $i$ that are at most $j$ for each such pair leading to the inner sum. We can further simplify as
    \begin{align*}
         \frac{1}{n^2}\frac{1}{2}\sum_{s=1}^{n-1} (n-s)(n-s-1)\alpha^s= \frac{1}{2}\sum_{s=1}^{n-1} (1-s/n)(1-s/n-1/n)\alpha^s.
         \end{align*}
         We may now take limits as $n\to \infty$. We can again apply dominated convergence using the convergence of $\alpha_n\to \alpha^*$, and so the limit of this sum is precisely 
         \begin{equation*}
             \frac{1}{2}\sum_{s=1}^{\infty} (\alpha^*)^s=\frac{\alpha^*}{2(1-\alpha^*)},
         \end{equation*}
         which gives the desired conclusion.
\end{proof}

\begin{lem}
\label{lem:fourth-moment-l-k-j-i}
    Let $\alpha_n\in [0,1)$ be a sequence that converges to some $\alpha^*<1$. Then 
    \begin{equation*}
        \sum_{1\leq i<j<k<\ell\leq n} \alpha_n^{(\ell-k)+(j-i)} = (1-o(1))\cdot \frac{1}{2}\cdot \left(\frac{\alpha^*}{1-\alpha^*}\right)^2\cdot n^2
    \end{equation*}.
\end{lem}
\begin{proof}
    We first find a simpler combinatorial way to count the relevant indices. Let $r=\ell-k\geq 1$, $s=j-i\geq 1$, and $t=k-j\geq 1$. This gives a natural bijection from the indices $1\leq i<j<k<\ell\leq n$ to the set of indices $1\leq i,s,r,t\leq n$ satisfying $i+s+r+t\leq n$. Moreover, for any fixed $s,r,t$ satisfying $s+r+t\leq n-1$, there are exactly $n-r-s-t$ valid choices for $i$ that can admit these gaps. In particular, we may write
    \begin{align*}
        \sum_{1\leq i<j<k<\ell\leq n} \alpha^{(\ell-k)+(j-i)} &= \sum_{i,s,t,r: i+s+t+r\leq n}\alpha^{s+r}\\
        &=\sum_{s,t,r: s+t+r\leq n-1}\alpha^{s+r}(n-s-t-r).
    \end{align*}
    If we define $u := r+s$, then there are $u-1$ valid choices of $r,s$ that yield any possible $u$, and so the sum becomes:
    \begin{align*}
        \sum_{u,t: 2\leq u\leq n-2,t+u\leq n-1}(u-1)\alpha^{u}(n-u-t)&=\sum_{2\leq u\leq n-2}(u-1)\alpha^u \left(\sum_{t=1}^{n-u-1} (n-u-t)\right)\\
        &=\sum_{2\leq u\leq n-2}(u-1)\alpha^u \frac{(n-u)(n-u-1)}{2}.
    \end{align*}
    At this point, we can expand in $n$ to find that
    \begin{equation*}
        \sum_{1\leq i<j<k<\ell\leq n} \alpha^{(\ell-k)+(j-i)} = \frac{n^2}{2} \sum_{u=2}^{n-2} (u-1) \alpha^u + O(n),
    \end{equation*}
    where we may use standard geometric bounds for the lower order contributions in $n$. For this dominant geometric term, we can again apply dominated convergence in the same way as before to find that
    \begin{equation*}
        \lim_{n\to \infty}\sum_{u=2}^{n-2} (u-1) \alpha^u = \left(\frac{\alpha^*}{1-\alpha^*}\right)^2.
    \end{equation*}
\end{proof}

%%%%%%%%%%%%%%%%%%%% NeurIPS Checklist %%%%%%%%%%%%%%%%%%%%

% \newpage
% \input{neurips/checklist.tex}

\end{document}